\documentclass[twoside]{article}

\usepackage[accepted]{aistats2025}

\usepackage[hidelinks]{hyperref}
\usepackage{mathtools}
\usepackage{float}
\usepackage{wrapfig}
\usepackage[utf8]{inputenc}
\usepackage{graphicx}
\usepackage[T1]{fontenc}    % use 8-bit T1 fonts
\usepackage{subfig}
\usepackage{booktabs} % For \toprule, \midrule and \bottomrule
\usepackage{amsmath, amsfonts, amsthm, amssymb} 
\usepackage{url}            % simple URL typesetting
\usepackage{microtype}      % microtypography
\usepackage{lipsum}
\usepackage{cleveref}
\usepackage{nicefrac}
\usepackage{adjustbox}
\usepackage{multirow}
\usepackage{pgfplotstable}

\usepackage{nicematrix} % for block matrices
\makeatletter
\newcommand{\longdash}[1][2em]{%
  \makebox[#1]{$\m@th\smash-\mkern-5mu\cleaders\hbox{$\mkern-2mu\smash-\mkern-2mu$}\hfill\mkern-5mu\smash-$}}
\makeatother
\newcommand{\omitskip}{\kern-\arraycolsep}

\usepackage[round]{natbib}

\usepackage[shortlabels]{enumitem}
\setlist{leftmargin=*, topsep=1pt}
\setlist[1]{labelindent=\parindent}

\usepackage{thmtools}
\usepackage{mdframed}
\mdfsetup{skipabove=10pt,skipbelow=5pt}

\declaretheoremstyle[mdframed={outerlinewidth=1pt,innertopmargin=5pt, innerbottommargin=2pt, backgroundcolor=gray!5, linecolor=gray!30}]{thmstyle}
\declaretheorem[style=thmstyle]{theorem}
\declaretheorem[style=thmstyle]{proposition}

\usepackage{algorithm}
\usepackage{soul}
\declaretheoremstyle{exstyle}

\usepackage{titlesec}
\usepackage[toc,page,header]{appendix}
\usepackage{minitoc}
\usepackage{mynotation}

\newcommand{\thelink}{\href{https://github.com/antonior92/advtrain-linear}{\texttt{github.com/antonior92/advtrain-linear}}}
\begin{document}

% Change equation spacing
\setlength{\abovedisplayskip}{3pt}
\setlength{\belowdisplayskip}{1pt}

\doparttoc % Tell to minitoc to generate a toc for the parts
\faketableofcontents % Run a fake tableofcontents command for the partocs

\twocolumn[
\aistatstitle{Efficient Optimization Algorithms for Linear Adversarial Training}

\aistatsauthor{ Ant\^onio H. Ribeiro \And Thomas B. Schön \And Dave Zachariah \And  Francis Bach}

\aistatsaddress{ Uppsala University \And Uppsala University \And Uppsala University \And  INRIA, PSL Research Univ.} ]

\begin{abstract}
  Adversarial training can be used to learn models that are robust against perturbations. For linear models, it can be formulated as a convex optimization problem.  Compared to methods proposed in the context of deep learning, leveraging the optimization structure allows significantly faster convergence rates. Still, the use of generic convex solvers can be inefficient for large-scale problems. Here, we propose tailored optimization algorithms for the adversarial training of linear models, which render large-scale regression and classification problems more tractable. For regression problems, we propose a family of solvers based on iterative ridge regression and, for classification, a family of solvers based on projected gradient descent. The methods are based on extended variable reformulations of the original problem. We illustrate their efficiency in numerical examples.
\end{abstract}

\part{} % Start the document part
\vspace{-50pt}

\section{\uppercase{Introduction}}

Adversarial training can be used to estimate models that are robust against perturbations~\citep{madry_towards_2018}.  It considers training samples that have been modified by an adversary, with the goal of obtaining a model that will be more robust when faced with newly perturbed samples. The training procedure is formulated as a min-max optimization problem, searching for the best model given the worst-case perturbations.

In linear models, adversarial training can be seen both as a type of regularization and as a type of robust regression.  This method is promising for estimating linear models: For the $\ell_2$-norm, adversarial training has properties similar to ridge regression, and for the $\ell_\infty$-norm, it is similar to Lasso and produces sparse solutions. Unlike the hyperparameters set in these conventional methods, however, the adversarial radius $\delta$ has a clear interpretation and---as proved by~\citet{ribeiro_regularization_2023,xie_high-dimensional_2024}---this parameter can be set independently of the noise levels in the problem. Indeed, for adversarial training, there exists a default choice of $\delta$ values that yields zero coefficients for random outputs and near-oracle prediction performance. Hence, it could be a powerful method in applications where cross-validation cannot be effectively used to choose the regularization parameter.  We illustrate this in Fig.~\ref{fig:adv-train-linear}(\emph{left panel}), where \emph{adversarial training with the default adversarial radius $\delta$} achieves a performance that is close to that of \emph{cross-validated Lasso} in diverse benchmarks. 
 
 \begin{figure*}[t]
    
    \centering
    \includegraphics[width=0.48\textwidth]{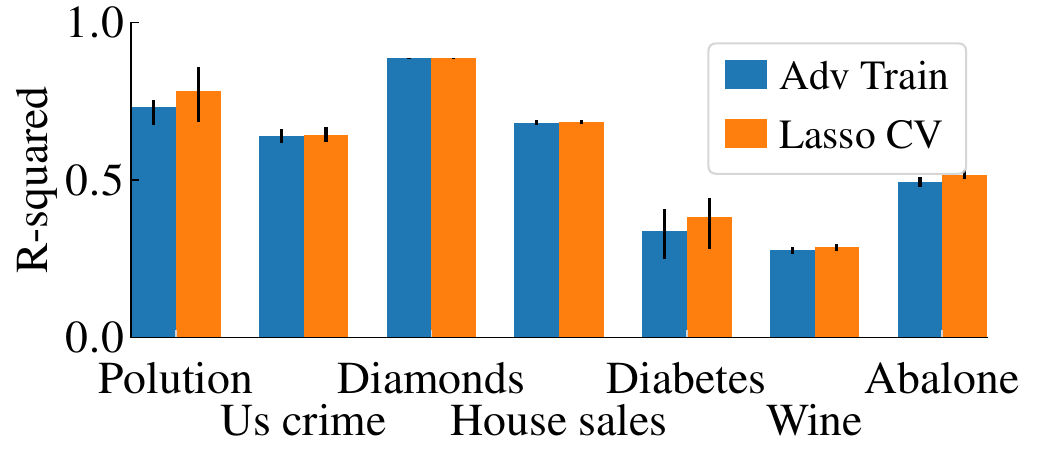}
    \includegraphics[width=0.48\textwidth]{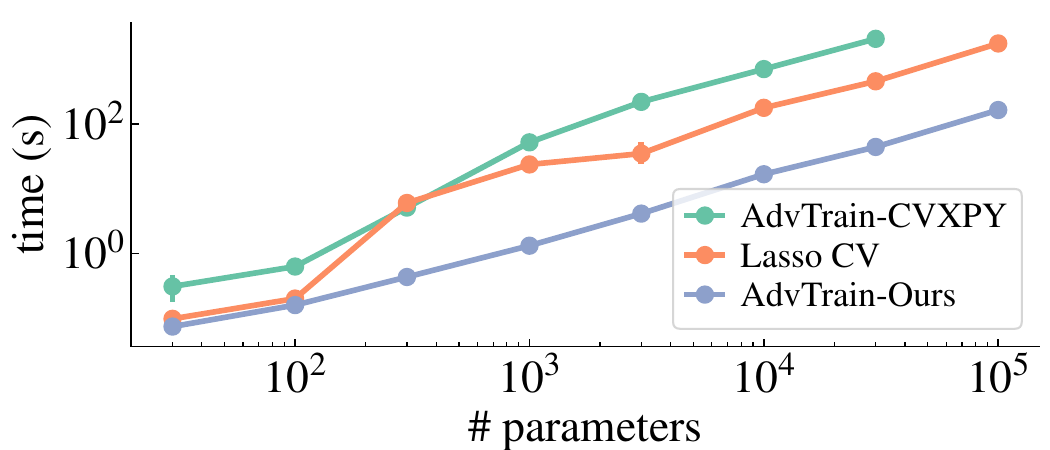}\vspace{-5pt}
    \caption{\emph{Adversarial training in linear regression.} \emph{Left}: we compare linear adversarial training ($\ell_\infty$-norm bounded attacks) using the \emph{default adversarial radius}---see~\citet{ribeiro_regularization_2023} and the description in~\Cref{sec:default-value}---and Lasso with parameters selected through \emph{cross-validation}.   \emph{Right}: the running times for prunned MAGIC dataset of \emph{adversarial training with our tailored solver}, \emph{with CVXPY}, and of \emph{cross-validation Lasso}, optimized as in~\citet{friedman_pathwise_2007}. 
    The error bars in (a) show the interquartile range obtained using bootstrap. In (b), we plot the median of 5 repetitions.}
    \label{fig:adv-train-linear}\vspace{-10pt}
\end{figure*}
 Consider a training dataset  $\{(\x_i, y_i)\}_{i =1}^n$ consisting of $\ntrain$ training datapoints of dimension $\R^\nfeatures \times \R$. Adversarial training in linear models finds the parameter vector $\param \in \R^{\nfeatures}$ by solving
\begin{equation}
\label{eq:advtrain}
   \min_{\param}\left(\frac1n \sum_{i=1}^\ntrain{\max_{\|\dx_i\| \le \delta} \ell\big(y_i, (\x_i + \dx_i)^\top\param\big)}\right), 
\end{equation}
where $\ell(\cdot, \cdot)$ is the loss and the adversarial radius $\delta$ gives the maximum magnitude of the adversarial perturbation. 
In this paper, we consider both classification problems and regression problems---with the logistic and squared error loss. In both cases, the function being minimized in~\eqref{eq:advtrain} is convex.

The two most direct options for solving \eqref{eq:advtrain} have drawbacks. The \emph{first option} is to use \emph{solvers proposed in the context of deep learning}, e.g.,~\citet{madry_towards_2018}. These approaches are scalable but suffer from \emph{slow convergence}, and therefore inaccurate solutions, since they do not exploit the structure of the optimization problem. The \emph{second option} is to use off-the-shelf convex solvers~\citep{diamond_cvxpy_2016}, which in contrast, usually produce accurate solutions due to their fast convergence. On the other hand, they offer limited scalability and are prohibitively slow for problems with many variables, as in the motivating example below. In this paper \emph{we propose tailored solvers that are accurate and efficient for high-dimensional problems, enabling new applications}. 

\textbf{Motivating example.} 
Consider phenotype prediction from genotype in the Diverse MAGIC wheat dataset~\citep{scott_limited_2021} from the National Institute for Applied Botany. The dataset contains the whole genome sequence data and multiple phenotypes for a population of 504 wheat lines. In this problem, our tailored solver converges in less than 5 minutes while the CVXPY solver~\citep{diamond_cvxpy_2016} did not manage to converge in one day.  Fig.~\ref{fig:adv-train-linear}(\emph{right panel}) shows running times for subsampled inputs with increasing sizes. Here the coefficient of determination $R^2$ is used to evaluate goodness-of-fit (higher is better): $\ell_\infty$-adversarial training with default parameter has a performance $R^2=0.35 $ (IQR 0.27-0.46) while cross-validated lasso leads to $R^2=0.38$ (IQR 0.27-0.41).\footnote{We use IQR for bootstrapped interquartile range.}

There are important reasons to be \emph{interested directly in linear models}. High-dimensional sparse problems, where the number of features far exceeds the number of observations, are common in fields like computational biology and genetics.
High variance and overfitting are major concerns, and \emph{simple but highly regularized approaches are often preferred}. For instance, Lasso is very frequently used for polygenic risk score~\citep{torkamani_personal_2018}, proteomics analysis~\citep{safo_derivation_2023}, and multivariate Mendelian randomization~\citep{burgess_robust_2020}. However, Lasso has limitations: while cross-validation can be effective for prediction, it often selects too many variables~\citep{buhlmann_statistics_2011}, especially in small datasets where cross-validation itself may be unreliable. Consequently, methods that bypass cross-validation, such as those implemented in polygenic risk score analysis packages like \texttt{lassosum} and \texttt{LDPred}~\citep{mak_polygenic_2017, vilhjalmsson_modeling_2015}, are often pursued. In light of these challenges, we propose that linear adversarial training could be a promising alternative, with the development of effective solution methods being critical to its success.

\textbf{Contributions.}
The contribution of this paper is to \emph{reformulate the problem and propose tailored efficient solvers for the minimization of Eq.~\eqref{eq:advtrain}}. We propose: (A) An \emph{augmentated} variable formulation that renders a smooth optimization problem in adversarially trained \emph{logistic regression}. We propose an efficient family of first-order algorithms based on this reformulation, allowing us to efficiently solve~\eqref{eq:advtrain} when $\ell(y, \hat{y}) = \log\left(1 + e^ {-y\hat{y}}\right)$. (B) A \emph{reformulation} of adversarially trained \emph{linear regression} that can be solved using iteratively reweighted least-squares problems. This reformulation allow us to efficiently solve~\eqref{eq:advtrain} when $\ell(y, \hat{y}) = \left(y - \hat{y}\right)^2$. We validate these algorithms with real and simulated data. The implementation is available in~\thelink.

\section{\uppercase{Related work}}
\label{related-work}

Adversarial attacks can significantly degrade the performance of state-of-the-art models and adversarial training has emerged as one of the most effective strategies. Relevant related work is discussed next.

\textbf{Adversarial  training for deep learning models.}
Traditional methods for generating \emph{adversarial attacks} involve solving an optimization problem. One approach is to maintain the model's prediction while minimizing the disturbance to the input. Examples include L-BFGS \citep{bruna_intriguing_2014}, DeepFool \citep{moosavi-dezfooli_deepfool_2016} and the C\&W method~\citep{carlini_towards_2017}. Another approach is to maximize the error with limited attack size, such as in the Fast Gradient Size Method~\citep{goodfellow_explaining_2015}, and PGD type of attacks~\citep{madry_towards_2018}. \emph{Aversarial training}  builds on the latter approach, resulting in a min-max  problem, as in~\eqref{eq:advtrain}, that can be solved backpropagating through the inner loop~\citep{madry_towards_2018}. Most improvements proposed in the context of deep learning try to better deal with the inner maximization problem. And, for the linear case, they do not provide improvements over the Fast Gradient Size Method (FGSM), since the inner problem can already be solved exactly. As we discuss in~\Cref{sec:deep-learning-comparison},  \emph{just adapting FGSM to the linear case has limitations in that it is hard to get accurate solutions due to its slow convergence rate.}

\textbf{Adversarial training in linear models.} Adversarial attacks in linear models have been intensely studied over the past 5 years. Linear adversarial training provides an alternative to Lasso and other sparsity inducing methods~\citep{ribeiro_regularization_2023,xie_high-dimensional_2024}. It also provides a simplified analysis for understanding the properties of deep neural networks: indeed, linear models have been used to explore the trade-off between robustness and performance in neural networks~\citep{tsipras_robustness_2019, ilyas_adversarial_2019} and to examine how overparameterization impacts robustness~\citep{ribeiro_overparameterized_2023}. Adversarial training has been studied through asymptotics in binary classification~\citep{taheri_asymptotic_2022} and linear regression~\citep{javanmard_precise_2020}, as well as in classification~\citep{javanmard_precise_2022} and random feature regression~\citep{hassani_curse_2022}. Gaussian classification problems have also been analyzed~\citep{dan_sharp_2020, dobriban_provable_2022}, alongside the effects of dataset size on adversarial performance~\citep{min_curious_2021} and $\ell_\infty$-attacks on linear classifiers~\citep{yin_rademacher_2019}.  
\emph{While the literature above focuses on the generalization properties of adversarial training in linear models, we propose new optimization algorithms to solve the problem.}

\textbf{Robust regression and square-root Lasso.}  Lasso~\citep{tibshirani_regression_1996} usage for sparse recovery depends on the knowledge of the noise variance. Square-root Lasso avoids this dependence~\citep{belloni_square-root_2011}. \citet{ribeiro_regularization_2023} and \cite{xie_high-dimensional_2024} show that $\ell_\infty$-adversarial settings also have the same desirable properties. Both methods can be framed within the robust regression framework~\citep{ribeiro_regularization_2023}. Still, \emph{having tailored solvers that work in high-dimensional settings yields a practical advantage to the adversarial training framework over square-root Lasso.}

\textbf{Sparsity-inducing penalties and convex optimization.} Our solutions are inspired by methods used in the optimization of other sparsity-inducing penalties. For instance, we use the ``$\eta$-trick'' (a.k.a.~iteratively reweighted least-squares) to solve adversarial-trained linear regression~\citep{daubechies_iteratively_2010,bach_eta-trick_2019, bach_optimization_2011}. Our convergence rate analysis uses tools from convex optimization~\citep{bubeck_convex_2015}. We also use conjugate gradient~\citep[Chap. 5]{nocedal_numerical_2006, golub_matrix_2012} to solve the linear systems that arise as subproblems. Finally, we implement variance reduction methods to improve our solver~\citep{gower_variance-reduced_2020}. \emph{We reformulate adversarial training in order to leverage these tools.}

\vspace{-7pt}

\section{\uppercase{Classification}}
\label{classification}
In this section, we focus on adversarial training for classification. More precisely, in solving \eqref{eq:advtrain} when ${y \in \{-1, +1\}}$ and $\ell(y, \hat{y}) = h(y\hat{y})$ for $h$ non-increasing, smooth and convex:
\begin{equation*}
    \min_{\param} \frac1n \sum_{i=1}^\ntrain{\max_{\|\dx_i\| \le \delta} h\big(y_i(\x_i + \dx_i)^\top\param\big)}\\
\end{equation*}
We focus on $h(z) = \log\left(1 + e^{-z}\right)$, but other alternatives are possible, i.e., $h(z) = (1 - z)_+^2$.
We first introduce a reformulation of this problem that turns it into a constrained smooth optimization problem (\Cref{sec:reformulation}) and propose the use of projected gradient descent to solve it. In~\Cref{sec:projection-step} we give more details on how to solve the projection step efficiently. In~\Cref{sec:deep-learning-comparison} we give a comparison with methods used in the context of deep learning.  Finally, in~\Cref{sec:improvements} we propose improvements of this base algorithm.

\textbf{Notation.} We use $\|\cdot\|$ to denote an arbitrary norm and ${\|\param\|_* =  \sup
\{\param^\top \x: \|\x\|\le 1\}}$ its dual norm.
The relevant pairs of dual norms for this papers are ${(\|\cdot\|_2, \|\cdot\|_2)}$ and ${(\|\cdot\|_1, \|\cdot\|_{\infty})}$.

\subsection{Smooth minimization formulation}
\label{sec:reformulation}
The following result allows us to solve a smooth constrained optimization problem rather than the original min-max optimization.

\begin{proposition}
    \label{thm:advtrain-classif-closeform}
Let $\rho > 0$ and $\ell(y, \hat{y}) = h(y\hat{y})$ for $h$ non-increasing, 1-smooth and convex. The optimization of \eqref{eq:advtrain} is  equivalent to
\begin{align}\vspace{-5pt}
\min_{(\param,t)} &~~\mathcal{R}(\param, t) \myeq  \frac{1}{n}\sum_{i=1}^n h(y_i \x_i^ \top \param - \rho t), \label{eq:constrained problem}\\
&\text{subject to: } \rho t \ge \delta \|\param\|_*.\nonumber
\end{align}
Moreover, the function $\mathcal{R}$ is $L$-smooth and jointly convex. And, $L \le\frac{1}{2} \lambda_{\max_{}}(\frac{1}{n}\sum_i \x_i\x_i^\top)$ if ${\rho^2 \le\lambda_{\max_{}}(\frac{1}{n}\sum_i \x_i\x_i^\top)}$.
\end{proposition}

This reformulation allows us to solve the problem efficiently using projected gradient descent (GD). 
\begin{algorithm}[H]
\caption{Smooth min. with projected GD}\label{alg:pgd}
\textit{Choose}  step size $\gamma$ \\
$\vv{w}^{(0)} \leftarrow  (\param^{(0)},t^{(0)})$ \\
\textbf{for $k = 1, 2 \dots$:}
\begin{enumerate}
    \item[]$\vv{w}^{(k)} \leftarrow  \proj_C\left(  \vv{w}^{(k-1)} - \gamma \nabla \mathcal{R}(\vv{w}^{(k-1)})\right)$
\end{enumerate}
\end{algorithm}
\vspace{-10pt}
The cost function $\mathcal{R}$ optimized in~\Cref{thm:advtrain-classif-closeform} is $L$-smooth (i.e., its gradient is $L$-Lipshitz continuous, that is, $\|\nabla \mathcal{R}(\vv{w}) - \nabla \mathcal{R}(\vv{z})\|_2 \le L \|\vv{w} - \vv{z}\|_2$). Let $k$ denote the number of iterations, projected GD (\Cref{alg:pgd}) has a convergence rate of $\bigO(L/k)$ for a step size $\gamma = 1/ L$, see~\citet[Theorem 3.7]{bubeck_convex_2015}. 
Let
$\vv{w}$ denote the extended variable in  $\R^{\nfeatures + 1}$, $C$ the set of feasible points, i.e., \[C = \{\vv{w} = (\param, t ):  \rho t \ge \delta \|\param\|_*\} \subset \R^{\nfeatures + 1},\]  and $\proj_C,$ the projection of a point into this set.

\begin{proof}[Proof of \Cref{thm:advtrain-classif-closeform}]
Eq. \eqref{eq:advtrain} is equivalent to 
\begin{equation}
\label{eq:reformulated-classification-problem}
\min_{\param}\frac{1}{n}\sum_{i=1}^n h(y_i \x_i^\top \param - \delta \|\param\|_*),
\end{equation}
for $\ell(y, \hat{y}) = h(y\hat{y})$ and $h$ non-increasing, smooth and convex. See ~\Cref{adv-train-reform} for a proof.  In turn, the minimization of \eqref{eq:reformulated-classification-problem} is equivalent to that of $\mathcal{R}(\param, t)$ subject to the constraint  $\delta t \ge \rho \|\param\|$. And since $\ell(\cdot)$ is convex and smooth, so  is $\mathcal{R}(\cdot, \cdot)$. Finally, the Hessian of $\mathcal{R}$ is:
\begin{equation}
\label{eq:bound-hessian}
\nabla^2 \mathcal{R} (\vv{w})\preccurlyeq \frac{1}{2}
\begin{bmatrix}
  \frac{1}{n} \sum_{i=1}^n \x_i \x_i^\top& 0\\
  0 & \rho^2 
\end{bmatrix}.
\end{equation}
And, since $L  \le \sup_{\vv{w}}\|\nabla^2 \mathcal{R}(\vv{w})\|_2$, we have
\begin{align*}
     L \le \frac{1}{2} \max\Big(\big\|\tfrac{1}{n}\sum_{i=1}^n \x_i \x_i^\top\big\|_2, \rho^2 \Big) \le \lambda_{\max}\Big(\tfrac{1}{n}\sum_{i=1}^n \x_i \x_i^\top\Big),
\end{align*}
% See for instance https://math.stackexchange.com/questions/1404043/norm-of-a-block-matrix
where the last inequality follows from the condition on $\rho$ and that $\|\cdot \|_2 =\lambda_{\max}(\cdot)$ for matrices.
\end{proof}

\subsection{Projection step}
\label{sec:projection-step}

The set $C$ is convex and we can define the following projection operator:
\begin{align*}
 \proj_C (\widetilde{\param}, \tilde{t}) &= \mathrm{arg}\min_{\rho  t \ge \delta \|\param\|_*\nonumber}  \frac{1}{2}\|\param  - \widetilde{\param}\|^2_2 + \frac{1}{2}(t - \tilde{t})^2.
\end{align*}
For $\ell_2$-adversarial and $\ell_{\infty}$-adversarial attacks the following propositions can be used for computing the projection. The proofs are given in~\Cref{proof-proj}.

\begin{proposition}[Projection step]
\label{thm:proj}
If  $(\widetilde{\param}, \tilde{t})   \in C$, the projection is the identity.
On the other hand, if  $(\widetilde{\param}, \tilde{t})  \not \in C$, let us denote $\proj_C (\widetilde{\param}, \tilde{t}) = (\param^{\proj}, t^{\proj})$. We have that:
\begin{enumerate}
\item For  $C = \{(\param, t): \rho t \ge \delta \|\param\|_2\}$, then
    $t^{\proj } = \delta\frac{\rho \|\widetilde{\param}\|_2 + \delta\tilde{t}}{\delta^2 + \rho^2}$ and $\param^ {\proj} = \frac{\rho t^{\proj}}{\delta\|\widetilde{\param}\|_2} \widetilde{\param}.$

    \item Now, for  $C = \{(\param, t): \rho t \ge \delta \|\param\|_1 \}$ .   Let $\left(\cdot \right)_+$ denote the operator  $\max (\cdot, 0)$ and define $\lambda$  as the solution to
    \begin{equation}
    \label{eq:piecewise-linear}
    \sum_{i = 1}^\nfeatures (|\widetilde{\beta}_i|  - \delta \lambda )_+  =  \tilde t +  \rho \lambda,
    \end{equation}
    which exists and is unique for $\tilde t>0$. Then,
    $t^{\proj }\mathord{=}\tilde t + \rho \lambda$ \text{and} ${\beta^{\proj}_i\mathord{=}\text{sign}(\widetilde{\beta}_i)\large(|\widetilde{\beta}_i|  - \delta \lambda \large)_+.}$
    \end{enumerate}
\end{proposition}
Here the projection step for $\ell_{\infty}$-adversarial attacks needs to solve the piecewise linear equation \eqref{eq:piecewise-linear}. This equation can be solved by sorting the elements of the sum and has the overall computational cost of $\bigO(p\log{p})$.\footnote{Our problem is similar to projecting into the $\ell_1$ ball. See \citet{condat_fast_2016} for a description of different algorithms for it. There are alternatives with an \emph{expected} complexity of $\bigO(p)$ that might be faster in practice but have worst-case complexity$\bigO(p^2)$.  Similar algorithms might apply here but they will not be considered in our developments.}  For $\ell_2$-adversarial attacks the projection has the computational cost of $\bigO(p)$.

\subsection{Pratical considerations and improvements on the basic algorithm}
\label{sec:improvements}
\begin{figure}[t]
    \centering
    \includegraphics[width=0.48\textwidth]{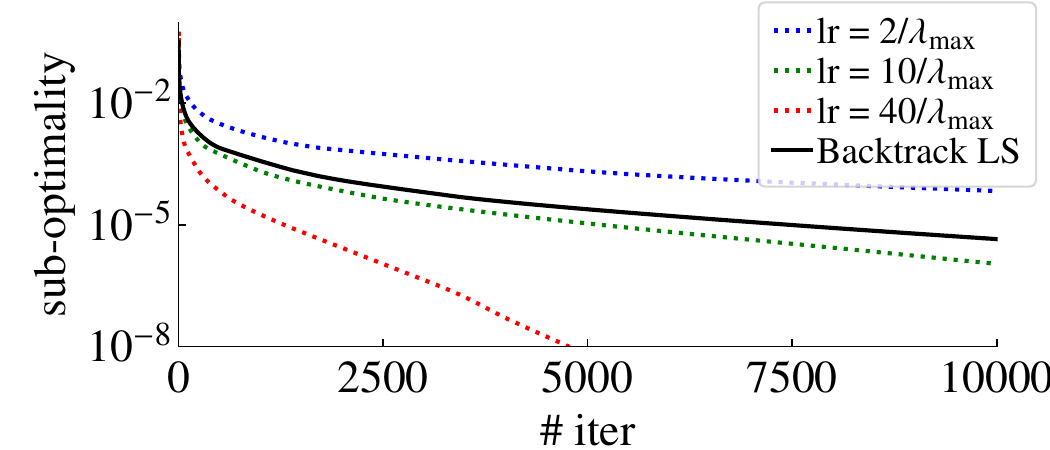}\\
    \includegraphics[width=0.48\textwidth]{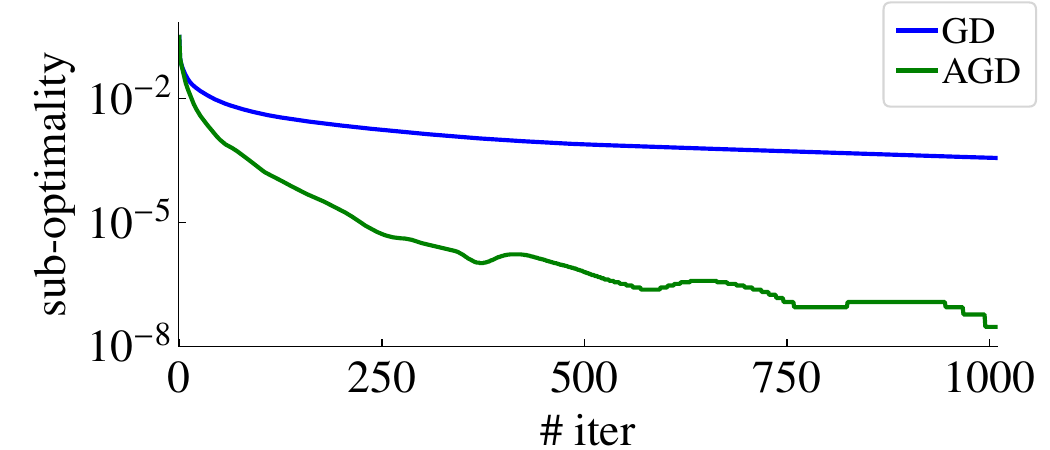}\\
    \includegraphics[width=0.48\textwidth]{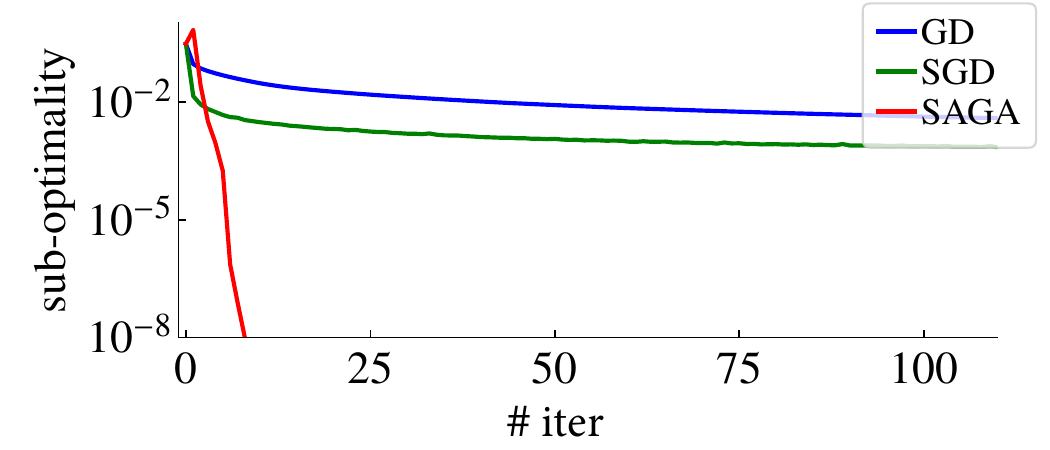}
    \caption{\emph{Convergence of Projected GD.} Sub-optimality \emph{vs} number of iterations in the MNIST dataset. \emph{Top:} we show the results for different step sizes. \emph{Middle:} we show the effect of momentum (with line search implemented for both methods). \emph{Bottom:} we compare GD, SGD and SAGA. Results for $\ell_\infty$-adv. training with $\delta=0.01$.  Here one epoch is a full pass throughout the dataset.}
    \label{fig:adv-train-pgd}
\end{figure}
In what follows, we describe practical details and improvements.   
%In (a), we discuss how to set the step size and the parameter $\rho$ in the basic algorithm. In (b), we discuss an alternative to a fixed step size, namely, backtracking line search. In (c), acceleration for a faster convergence rate. In (d), stochastic gradient descent, and in (e) the use of variance reduction.
These variations that we proposed, implemented and evaluated are described in \Cref{tab:summary-algorithm-classification}.  We also report their convergence rates and costs per interaction. Pseudo-algorithms are provided in \Cref{adversarial-train-logist-implementations}.

\begin{table*}
    \centering
    \caption{\emph{Proposed algorithms for classification.} Variance reduction is abbreviated as \emph{VR}, Acceleration as \emph{Acc.}we consider the computational cost of the projection step $\bigO(p\log{p})$.} 
    \begin{tabular}{l|l|ccc|cc}
    \hline
        & Algorithm  & Line search & Acc. &  VR & Cost/iter ($\nabla$ + $\proj$) & Convergence Rate  \\
         \hline
        \multirow{3}{*}{Determin.}  & FGSM with GD & $\times$ & $\times$  & - & $\bigO(\nfeatures \ntrain) $ & $\bigO(R/\sqrt{k})$\\
        & GD \textbf{(Ours)} &  \checkmark & $\times$& - & $\bigO(\nfeatures \ntrain) + \bigO(\nfeatures \log{\nfeatures})$ & $\bigO(L/k)$\\
        & AGD  \textbf{(Ours)} &   \checkmark  & \checkmark  & - & $\bigO(\nfeatures \ntrain) + \bigO(\nfeatures \log{\nfeatures})$ & 
        $\bigO(L/k^2)$\\
          \hline
        \multirow{3}{*}{Stochastic}  & FGSM with SGD & $\times$ & $\times$ &   - & $\bigO(\nfeatures) $ & $\bigO(R/\sqrt{k})$\\
        & SGD \textbf{(Ours)}&  $\times$ & $\times$ & $\times$  & $\bigO(\nfeatures) + \bigO(\nfeatures \log{\nfeatures})$ & $\bigO( (L+n)/\sqrt{k})$\\
        & SAGA \textbf{(Ours)} & $\times$ & $\times$ &   \checkmark  & $\bigO(\nfeatures) + \bigO(\nfeatures \log{\nfeatures})$ & $\bigO((L+n)/k)$\\
        \hline
    \end{tabular}\\

    \begin{minipage}{1\textwidth}\linespread{0.7}\selectfont
        {\scriptsize  \textbf{Notes:} $L$ and $R$ are the smoothness and Lipschitz constant of $\mathcal{R}$ in \eqref{eq:constrained problem}. For FGSM with GD it is not possible to use efficient line search methods that are based on curvature conditions. The proof of the convergence rate of GD, AGD, SGD is given in~\citep{bubeck_convex_2015}, Chapters 3 and 6. The proof for SAGA is given in~\citep{defazio_saga_2014}. }
    \end{minipage} \vspace{-2pt}
    \label{tab:summary-algorithm-classification}
\end{table*}

\textbf{(a) Choosing step size.} \Cref{thm:advtrain-classif-closeform} yields practical choices of $\gamma$. Notice that $2 L = \lambda_{\max_{}}(\frac{1}{n}\sum \x_i\x_i^\top)$ yielding  $\gamma = 1/L =  2/ \lambda_{\max_{}}(\frac{1}{n}\sum \x_i\x_i^\top)$. In our implementation, we approximately compute the maximum eigenvalue using 10 iterations of the power method \citep[Section 7.3.1]{golub_matrix_2012}. Motivated by the proof of~\Cref{thm:advtrain-classif-closeform}, we also choose $\rho = \frac{1}{\sqrt{\ntrain}}\sum_i^n\|\x_i\|_2$. This guarantees the conditions of the theorem are satisfied. It also guarantees that the trace of the matrix in the left-hand-side of \Cref{eq:bound-hessian} is the same as the trace of the empirical covariance matrix $\tfrac{1}{n}\sum_{i=1}^n \x_i \x_i^\top$. As an alternative for a fixed step size, we propose the use of backtracking line search (LS), which we will describe next. Notice, that LS will not work for stochastic methods in (d) and (e), hence, the considerations here are still useful.

\textbf{(b) Backtracking line search.} This method is used to determine the step size. The method involves starting with a relatively large estimate of the step size for movement along the line search direction, and iteratively shrinking the step size (i.e., ``backtracking'') until a decrease of the objective function is observed that adequately corresponds to the amount of decrease that is expected \citep[Chapter 3]{nocedal_numerical_2006}. More precisely we use the same backtracking line search used by~\citet{beck_fast_2009}, it considers the following quadratic approximation around a point $\vv{w}$:
$\widetilde{\mathcal{R}}_{\vv{w}}(\vv{w}^+; \gamma) = \mathcal{R}(\vv{w}) + (\vv{w}^+ - \vv{w})^\top\nabla \mathcal{R}(\vv{w}) + \frac{1}{2\gamma} \|\vv{w}^+ - \vv{w}\|_2^2.$
accepting the next point $\vv{w}^+$ only if $\mathcal{R}(\vv{w}^+) \le \widetilde{\mathcal{R}}_{\vv{w}}(\vv{w}^+; \gamma)$. If a point is rejected, the learning used in the gradient step is divided by 2, consecutively, until the condition is satisfied.

\textbf{(c) Accelerated projected gradient descent.} Accelerated gradient descent is a popular improvement of the standard gradient descent. It uses information from the previous iteration to improve the convergence of gradient descent.
The cost function $\mathcal{R}$ optimized in~\Cref{thm:advtrain-classif-closeform} is $L$-smooth and step size $\gamma = 1/ L$, the projected gradient descent has a worst-case complexity of $\bigO(L/k)$.  The projected accelerated gradient descent(AGD) considers the addition of momentum for $\gamma = 1/ L$ and  has a worst-case complexity of $\bigO(L/k^2)$. The analysis of the algorithm follows similar steps as that of FISTA~\citep{beck_fast_2009}. The effect of using accelerated projected gradient in practice is illustrated in \Cref{fig:adv-train-pgd}(\emph{middle panel}).

\textbf{(d) Stochastic gradient descent.}
\label{sec:stochastic}
The stochastic gradient descent (SGD) method avoids the heavy cost per iteration of GD by using, in each update, one randomly selected gradient instead of the full gradient.
That is, let us define $f_i(\vv{w}) = \ell(y_i \x_i \param - \rho t)$, we can write $\mathcal{R}(\vv{w}) =  \frac{1}{n}\sum_{i=1}^n f_i(\vv{w})$ and its gradient $\nabla \mathcal{R}(\vv{w}) =    \frac{1}{n}\sum_{i=1}^n \nabla f_i(\vv{w})$. SGD  use one randomly selected $\nabla f_{i^{k}}$ at each iteration instead of the summation of all terms.    \emph{Batch implementations} were also implemented and tested. See~\Cref{fig:varying-batch-size} in the Appendix.

\textbf{(e) Variance reduction.} Variance reduction methods keep an estimate $g_k$ of the gradient, with the estimate having a variance that converges to zero \citep[see, e.g.,][and references therein]{gower_variance-reduced_2020}.  
For gradient descent methods we have $\Exp{\nabla f_i (\vv{w})} = \nabla \mathcal{R}(\vv{w})$. Having an unbiased estimate, however, is not sufficient to guarantee the convergence of SGD with fixed step size $\gamma$, since $\nabla f_i (\vv{w})$ has a variance that does not go to zero. While one can use sequence of decreasing step size, however, it can be difficult to tune this sequence in a way that this decreasing is not too fast or slow to guarantee a good convergence.

Variance reduction methods provide a solution to this problem. In our implementation, we use SAGA (stochastic average gradient `am\'elioré') method~\citep{defazio_saga_2014}. SAGA has several properties that make it attractive to our problem. Our problem is not necessarily strongly convex, and SAGA has proven convergence for non-strongly-convex functions with rate $\bigO(1/k)$ (which improves the $\bigO(1/\sqrt{k})$ rate obtained for gradient descent). Another advantage is that SAGA automatically adapts to $\mu$-strongly convex function, obtaining a linear convergence rate  $\bigO((1 - \mu/L)^k)$.  Hence, it can benefit from the curvature when it is available which can make it significantly faster in practice when compared to SGD, as we can see in~\Cref{fig:adv-train-pgd}(\emph{bottom panel}) where it looks like SAGA converges with linear rate of convergence (it is a line in the log-linear plot).

\subsection{Comparisson with deep learning adversarial training}
\label{sec:deep-learning-comparison}
Here we compare our method with adversarial training methods proposed originally for deep learning.  

\textbf{FGSM~\citep{goodfellow_explaining_2015}.} Let ${J_i(x_i) =  h\big(y_i\x_i^\top\param\big)}$ the Fast Gradient Sign Method (FGSM) result in the adversarial perturbation as ${\dx_i = \delta \cdot \sign \nabla_x J(\x_i)}$.
In \Cref{closed-form-fgsm}, we show that i) using FGSM~\citep{goodfellow_explaining_2015} to compute the adversarial attack yields the optimal solution for the inner loop when the model is linear and $\ell_\infty$ perturbations are used, and ii) using FGSM to compute the inner loop gives rise to the following optimization problem:
\begin{equation}
\min_{\param}\frac{1}{n}\sum_{i=1}^n h(y_i \x_i^\top \param - \delta \|\param\|_{1}).
\end{equation}
The cost function is not smooth in $\param$. On the one hand, gradient descent has a better rate of convergence for smooth functions (when compared to $O(\frac{1}{\sqrt{k}})$ for non-smooth function). On the other hand, the function being smooth allows for many of the improvements from \Cref{sec:improvements} to be applicable. This results in both worse results in practice (see \Cref{fig:comparisson-fgsm}) and also the worst theoretical convergence rate (See \Cref{tab:summary-algorithm-classification}). 

\begin{figure}
    \centering
    \includegraphics[width=0.48\textwidth]{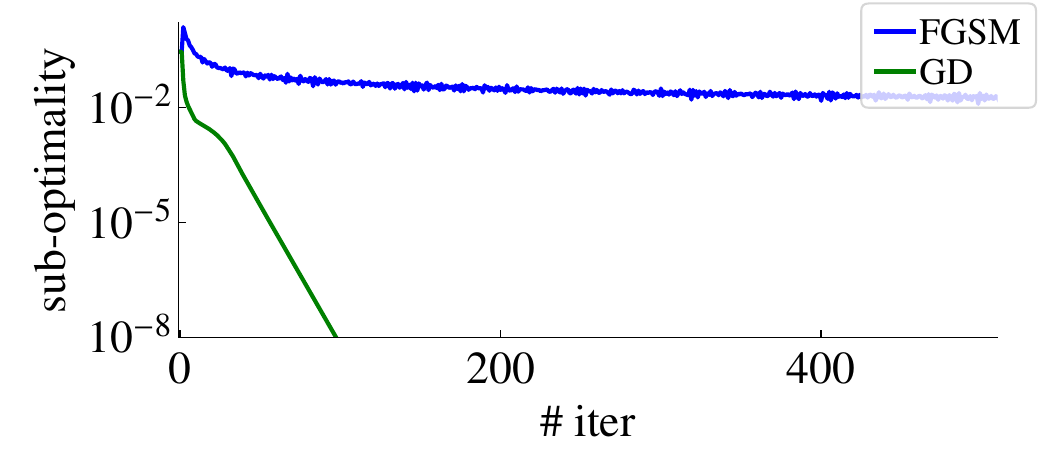}\\
    \includegraphics[width=0.48\textwidth]{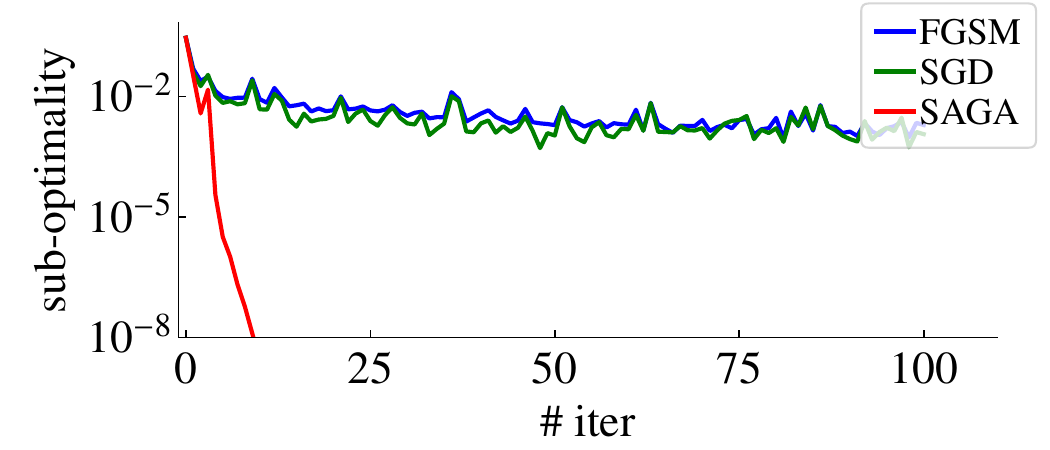}\vspace{-4pt}
    \caption{\emph{Comparison with FGSM.} Sub-optimality \emph{vs}  iterations.  On the \emph{top}, we compare our methods with the gradient descent implementation of FGSM.  On the \emph{bottom}, with the stochastic implementation of FGSM. Results for $\ell_\infty$-adversarial training ($\delta=0.01$).}
    \label{fig:comparisson-fgsm}\vspace{-10pt}
\end{figure}

\textbf{PGD~\citep{madry_towards_2018}} This popular approach makes use of projected gradient descent to solve the inner maximization problem, and the stochastic gradient descent (or alternatives) to solve the outer loop (having to backpropagate over the inner loop). In the linear case, however, multiple interactions are not needed to solve the inner loop, since we can solve it in one iteration. Hence, this method reduce to the same as applying FGSM. The same holds true for many of the other methods proposed in the context of deep learning.

\section{\uppercase{Regression}}
\label{regression}
In this section, we focus on adversarial training for regression (with a squared loss). In this case, the adversarial training problem can be formulated as
\begin{equation}
    \label{eq:adv-train-regression}
    \min_{\param} \frac1n \sum_{i=1}^\ntrain{\max_{\|\dx_i\| \le \delta} \big(y_i- (\x_i + \dx_i)^\top\param\big)^2}.
\end{equation}
Here, we propose the use of \Cref{alg:irrr} for solving it. We derive it as follows: we first introduce an extended formulation of this problem that is jointly convex and for which a blockwise coordinate descent optimization converges (\Cref{blockwise-coordinate-descent}). Next, we show that writing down the closed formula solutions for each minimization problem in the blockwise coordinate descent yields an iterative reweighted ridge regression algorithm (\Cref{irr}). Finally, in~\Cref{sec:improvements-irr} we propose improvements in this base algorithm.

\begin{algorithm}[H]
    \caption{Iterative Reweighted Ridge Regression}\label{alg:irrr}
    \textbf{Initialize:} sample weights $w_i\leftarrow 1, i = 1, \dots, \ntrain$.\\
    \phantom{\textbf{Initialize:}} parameter weights $\gamma_i\leftarrow 1, i = 1, \dots, \nfeatures$.
    
    \textbf{Repeat:}
    \begin{enumerate}
    \item \emph{Solve} reweighted ridge regression:\vspace{-5pt}
    \begin{equation*}
        \widehat{\param} \leftarrow\text{arg}\min_{\param} \sum_{i=1}^{\ntrain} w_i(y_i - \x_i^\top\param)^2 + \sum_{j=1}^{\nfeatures}  \gamma_j\beta_j^2\vspace{-5pt}
    \end{equation*}
    \item \emph{Update} weights:~~$\boldsymbol{w}, \boldsymbol{\gamma}\leftarrow \mathtt{UpdateWeights}(\widehat{\param})$
    \item \emph{Quit} if \texttt{StopCriteria}.
    \end{enumerate}
\end{algorithm}

\vspace{-10pt}
\subsection{Blockwise coordinate descent}
\label{blockwise-coordinate-descent}

The following result allows us to solve an easier optimization problem (although higher dimensional) rather than the original min-max optimization.
We present the result for the $\ell_{\infty}$-adversarial attacks here. The algorithm for $\ell_2$-adversarial attacks is similar and described in~\Cref{ell2-attacks}.\pagebreak[4]
\begin{proposition}
    \label{thm:advtraining-closeform}
    The minimization \eqref{eq:adv-train-regression}, for the $\ell_\infty$-norm (i.e., ${\|\dx \|_{\infty} \le \delta}$) is exactly equivalent to\vspace{-3pt}{\scriptsize
    \begin{align} \label{original_optimization}
        \min_{\param,\boldsymbol{\eta}^{(i)}} ~&~\scriptsize\sum_{i=1}^n\frac{\left(y_i - \x_i^\top\param\right)^2}{\eta^{(i)}_0} + \delta\sum_{j=1}^\nfeatures \sum_{i=1}^n\frac{1}{\eta^{(i)}_j}\beta_j^2 + \epsilon  \sum_{j=1}^\nfeatures \sum_{j=1}^{n}\frac{1}{\eta_j^{(i)}}.\\
    \mathrm{subject~to}&~~~\boldsymbol{\eta}^{(i)}> 0 ~~\mathrm{and}~~\|\boldsymbol{\eta}^{(i)}\|_1 = 1, i = 1, \dots, n. \nonumber
    \end{align}}
    when  $\epsilon = 0$. When $\epsilon >0$, the solution tends to that of  \eqref{eq:adv-train-regression} as $\epsilon \rightarrow 0$.  The function being minimized, denoted $\mathcal{G}_{\infty}$, is a jointly convex function in $\param, \boldsymbol{\eta}^{(1)}, \dots\boldsymbol{\eta}^{(n)}$.
\end{proposition}
Now, let us use $\Delta$ to denote the simplex $\Delta=\{\boldsymbol{\eta} \in \R^{p + 1}: \eta_t > 0, \sum_t \eta_t = 1\}$.
The following blockwise coordinate descent algorithm converges to the solution of the problem:\\
\begin{minipage}[t]{0.99\linewidth}
\begin{spreadlines}{-0.5em}% tweak 
\small
\begin{align}
        \widehat{\param} &\leftarrow \text{arg} \min_{\param} \mathcal{G}_{\infty}(\param, \widehat{\boldsymbol{\eta}}^{(1)},\dots, \widehat{\boldsymbol{\eta}}^{(n)}), \label{eq:block_1}\\
    \widehat{\boldsymbol{\eta}}^{(1)},\dots, \widehat{\boldsymbol{\eta}}^{(n)} &\leftarrow \text{arg} \min_{\boldsymbol{\eta}^{(i)} \in \Delta, \forall i} \mathcal{G}_{\infty}(\widehat{\param}, \boldsymbol{\eta}^{(1)},\dots, \boldsymbol{\eta}^{(n)}). \label{eq:block_2}
    \end{align}
\end{spreadlines}
\end{minipage}

The convergence of the above algorithm follows from results about the convergence of block coordinate descent for problems that are jointly convex and differentiable with unique solution along each direction~\citep[Section 8.6]{luenberger_linear_2008}. This is guaranteed in our case for $\epsilon > 0$. In \Cref{irr} we will present closed formula expression for each minimization problems. But, before that, we present the proof for~\Cref{thm:advtraining-closeform}.

\begin{proof}[Proof of~\Cref{thm:advtraining-closeform}] We do it in \textit{three steps}: \emph{first}, we prove in \Cref{adv-train-reform} that \eqref{eq:adv-train-regression} is equivalent to 
\begin{equation}
    \label{eq:advtraining-closeform-linf}
    \min_{\param} \sum_{i=1}^n\left(|y_i - \x_i^\top\param| + \delta\|\param\|_1\right)^2.
\end{equation}\vspace{-2pt}
\textit{Second}, we show  (using \Cref{thm:eta-trick}) that $\left(|y_i - \x_i^\top\param| + \delta\|\param\|_1\right)^2$ is equal to:
\[\footnotesize  \min_{\boldsymbol{\eta}^{(i)} \in \Delta_{(p+1)}^n}\frac{\left(y_i - \x_i^\top\param \right)^2}{\eta^{(i)}_0}  + \delta \sum_{j=1}^p\frac{\beta_j^2}{\eta^{(i)}_j}+ \epsilon \sum_{t=0}^{T-1} \frac{1}{\eta_t}\]
This holds either either $\epsilon = 0$ or when $\epsilon \rightarrow 0$. Hence, we can rewrite our problem as:
\begin{equation}
\label{eq:advtraining-closeform-linf-eta-trick}
\min_{\param} \min_{\substack{\boldsymbol{\eta}^{(i)}\in \Delta_{(p+1)} \\i=1, \dots, n}}\mathcal{G}_{\infty}.
\end{equation}
In our \emph{third} and final step, we prove that $\mathcal{G}_{\infty}$ is jointly convex (see \Cref{joint-convexity}). Since this is a jointly convex function, \eqref{eq:advtraining-closeform-linf-eta-trick} is equivalent to \eqref{original_optimization}.\end{proof}\vspace{-10pt}
\subsection{Iterative ridge regression algorithm}\vspace{-5pt}
\label{irr}
Now, we discuss the solution of the sub-optimization problems that arise from the block optimization. If we define $w_i = \frac{1}{\widehat{\eta}^{(i)}_0}, 
    \gamma_j =  \sum_{i=1}^n\frac{1}{\widehat{\eta}^{(i)}_j}$, and $C = \epsilon  \sum_{j=1}^\nfeatures \sum_{j=1}^{n}\frac{1}{\eta_j^{(i)}}$ we have that
    \[\mathcal{G}_{\infty}(\param, \widehat{\boldsymbol{\eta}}^{(1)},\dots, \widehat{\boldsymbol{\eta}}^{(n)})=\sum_{i=1}^{\ntrain} w_i(y_i - \x_i^\top\param)^2 + \sum_{j=1}^{\nfeatures}  \gamma_j\beta_j^2 + C.\] Hence, equation \eqref{eq:block_1} in the block coordinate optimization is equivalent to the first step in \Cref{alg:irrr}. And we define $\mathtt{UpdateWeights}$ to be
\begin{equation}
\label{eq:update_weights}\small
w_i  \leftarrow \frac{|y_i - \x_i^\top\widehat \param| + \|\param\|_1}{|y_i - \x_i^\top\widehat \param|}\text{ and }\gamma_j \leftarrow \sum_{i=1}^n\frac{|y_i - \x_i^\top\widehat \param| + \|\param\|_1}{|\widehat \beta_j|}
\end{equation}
This expression comes from the closed-formula solution for
$\min_{\boldsymbol{\eta}^{(i)} \in \Delta, \forall i} \mathcal{G}_{\infty}(\widehat{\param}, \boldsymbol{\eta}^{(1)},\dots, \boldsymbol{\eta}^{(n)}),$ obtained using \Cref{thm:eta-trick} bellow---sometimes referred to as  $\eta$-trick \citep{bach_eta-trick_2019, bach_optimization_2011}. 
\begin{proposition}[$\eta$-trick]
\label{thm:eta-trick}
For any $\epsilon > 0$ and any $a_t$ we have that for $\tilde a_t =\sqrt{a_t^2 + \epsilon}$:
    \begin{equation*}
    \left(\sum_{t=1}^n \tilde a_t\right)^2 = \min_{\boldsymbol{\eta} \in \Delta}  \sum_{t=0}^{T-1} \frac{a_t^2}{\eta_t} + \epsilon \sum_{t=0}^{T-1} \frac{1}{\eta_t}, 
    \end{equation*}
    which is minimized \emph{uniquely} at  
    \[\widehat \eta_t = \frac{\tilde a_t}{\sum_{t} \tilde a_t}, t = 0, \dots\]
Furthermore, if $a_t > 0$, for $t = 1, \dots, T$ then the result also holds for $\epsilon = 0$.
\end{proposition}
This result is proved in \Cref{proof-eta-trick}. Here expression \eqref{eq:update_weights} is for $\epsilon = 0$ and $a_0 = |y_i - \x_i^\top\widehat \param|$ and $a_i = |\widehat \beta_i|$. In practice, we might use a small $\epsilon> 0$ to guarantee that the denominator is not zero (see practical considerations), a similar expression follows from the proposition in this case.

\subsection{Pratical considerations and improvements on the basic algorithm}
\label{sec:improvements-irr}

\textbf{Numerical stability.} In practice,  using $\epsilon>0$ in \Cref{thm:eta-trick} is important to avoid numerical problems and to guarantee convergence.  Otherwise, when the coefficients $a_t$ are close to zero, we could obtain very large values for $1/\eta$ and an ill-conditioned problem. This happens often during adversarial training that both can yield solutions that interpolate the training data, i.e., $|y_i - \x_i \param| = 0$ for all $i$---see~\citet{ribeiro_regularization_2023}.  And also yields sparse solutions---i.e., $\param_j = 0$ for some values of $j$---for $\ell_\infty$-adversarial attacks.  In our experiments, we use $\epsilon=10^{-20}$.

\textbf{Computational cost.} We suggest the use of different implementations depending on weather $n > p$ or not.
Let us introduce the notation $\mm{W} = \text{diag}(w_1, \dots, w_n)$
and $\mm{\Gamma} =\text{diag}(\gamma_1, \dots, \gamma_p)$.  The weighted ridge regression (Step 1 in \Cref{alg:irrr}) has a closed formula solution
\begin{align}
\label{eq:ridge_solution}
\param^{\text{ridge}} &= (\mm{X}^\top \mm{W} \mm{X} + \delta \mm{\Gamma})^{-1} \mm{X}^\top \mm{W}  \vv{y} \\&\labelrel{=}{ridge_solution:inversionlemma} \mm{\Gamma}^{-1} \mm{X}^T (\mm{X} \mm{\Gamma} \mm{X}^\top  + \delta  \mm{W}^{-1})^{-1} \vv{y}.
\end{align}
Where the alternative formula \eqref{ridge_solution:inversionlemma} is obtained using the matrix inversion lemma. The computational cost of computing the solution using the first formulation (using Cholesky decomposition for inverting the matrix) is $\bigO(\nfeatures^2\ntrain + \nfeatures^3)$, while, using the alternative formulation $\bigO(\ntrain^2\nfeatures + \ntrain^3)$. Hence, one should use the first when  ${\nfeatures < \ntrain}$ and the second when ${\nfeatures > \ntrain}$. See Appendix~\ref{CVXPY} for a comparison with other alternatives.

\begin{figure}
    \centering
    \includegraphics[width=0.48\textwidth]{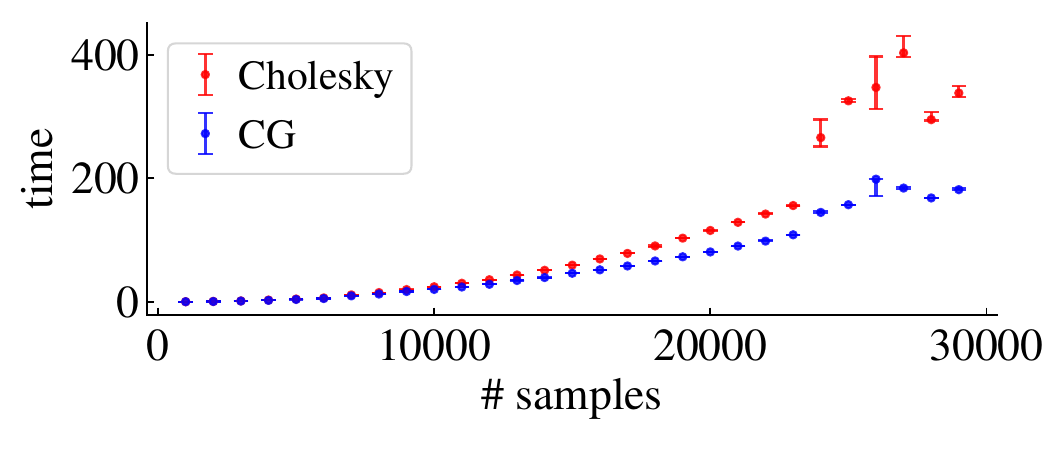}\vspace{-5pt}
    \includegraphics[width=0.48\textwidth]{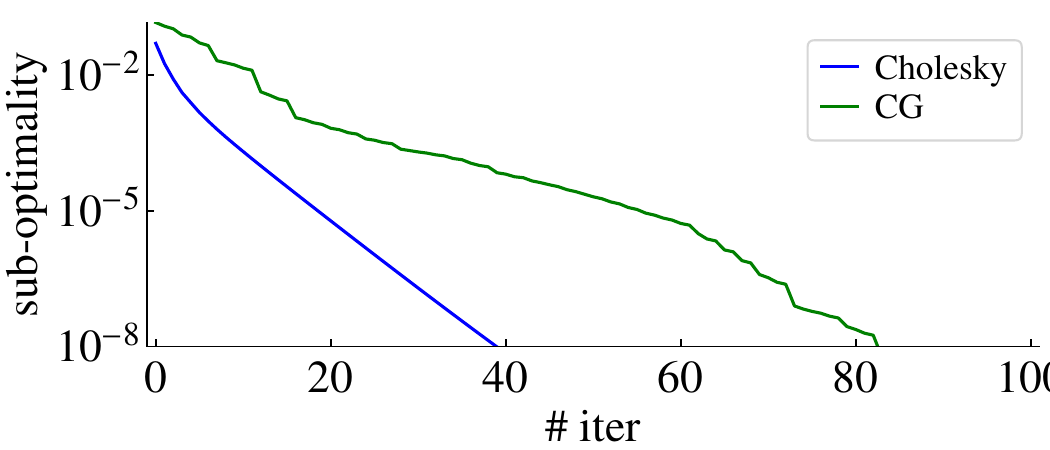}
    \caption{\emph{Conjugate Gradient.} \emph{Top:} execution time of 100 iterations in Abalone dataset.  \emph{Bottom:} convergence---suboptimality \emph{vs} iteration.}
    \label{fig:adv-train-icg}
\end{figure}
\textbf{Conjugate gradient implementation.}
As an improvement of~\Cref{alg:irrr}, we propose the use of conjugate gradient (CG) to approximately solve the reweighted ridge regression problem. As we discussed, the computation cost of the first step (solving ridge regression) of \Cref{alg:irrr} is cubic in $(\ntrain, \nfeatures)$ while the computational cost of the second step (computing weights) is only quadratic. A natural idea is to solve the ridge regression subproblem approximately with a computational cost that is also quadratic. In this manner, the computational effort is saved by not solving exactly a subproblem that needs to be updated.

The solution of~\eqref{eq:ridge_solution} is equivalent to the solution of the linear system $\mm{A} \param = \vv{b}$ for
\[\mm{A} = \mm{X}^\top \mm{W} \mm{X} + \delta \mm{\Gamma},~~ \vv{b} = \mm{X}^\top \mm{W}  \vv{y}\] as we mentioned in the previous section, the cost of solving this linear system exactly---for instance, using Cholesky or SVD decomposition---is $\bigO(\nfeatures^2\ntrain + \nfeatures^3)$. We propose the use of preconditioned conjugate gradient descent as an alternative for solving it (for large-scale problems). The algorithm  converges to the solution as \citep[p.117]{nocedal_numerical_2006}:
\[\|\param^{(k)} - \param^*\|_{\mm{A}} \le 2 \left(\frac{\sqrt{\kappa} -1} {\sqrt{\kappa} +1}\right)^k\|\param^{(0)} - \param^*\|_{\mm{A}}.\]
Where denote by $\kappa$ the condition number of the matrix $\mm{M}^{-1} \mm{A}$ and $\mm{M}$ is the  preconditioning matrix.
Our choice of preconditioner in the numerical experiments is the diagonal elements of $\mm{A}$, i.e.,
$\mm{M} = \text{diag}{\mm{A}}$. This way $\mm{M}^{-1}$ is inexpensive to compute keeping the overall cost per CG iteration low, while still improving convergence, by improving the conditioning.

\begin{figure} 
    \centering
    \includegraphics[width=0.48\textwidth]{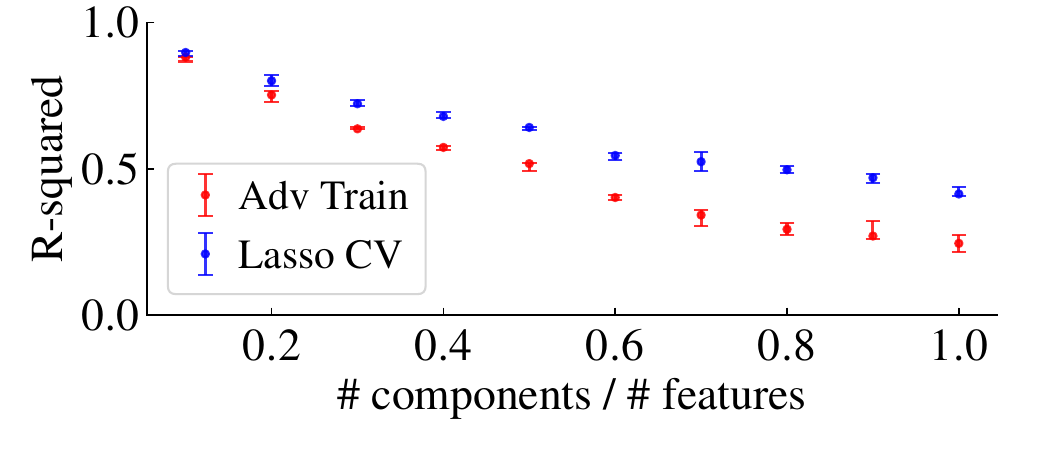}
    \includegraphics[width=0.48\textwidth]{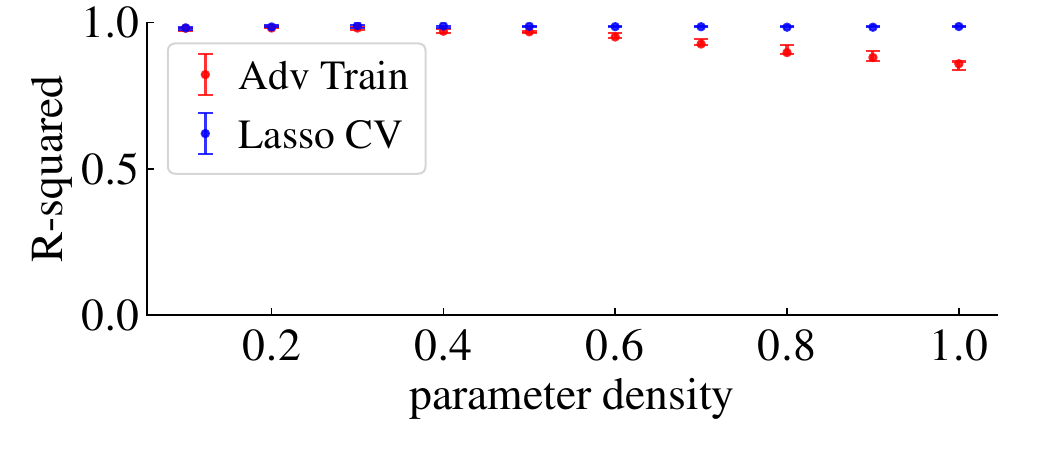}\vspace{-5pt}
    \caption{\emph{Test performance.} \emph{Top:} Spiked eigenvalue models: performance \emph{vs} fraction  of relevant eigenvalues components. \emph{Bottom:} Sparse parameter model: performance \emph{vs} fraction of non-zero parameters. $R$-squared is the coefficient of determination (higher is better).}
    \label{fig:syntetic_data} 
\end{figure}

\section{\uppercase{Numerical experiments}}
\label{numerical-experiments}

The numerical experiments aim to evaluate the \textbf{convergence} (suboptimality \emph{vs} iteration); \textbf{execution time}; \textbf{test performance}; and, \textbf{adversarial performance}. 

 We evaluate the test performance using the default value $\delta$ equals to the 95\% percentile of ${\|\vv{\varepsilon}^\top X\|/\|\vv{\varepsilon}\|_1}$ for $\vv{\varepsilon}$ standard Gaussian (this can be easily computed by simulation). This choice follows from~\citet{ribeiro_regularization_2023} and yields zero coefficients for random outputs and near-oracle prediction error (see \Cref{sec:default-value}). For adversarial performance, we evaluate in the test set, but consider that the test set has been perturbed by adversarial attacks with an adversarial radius of $\delta^{\rm eval}$ (not necessarily the same as in the training $\delta$). 

 In our evaluation, we consider 13 \emph{real datasets} (7 for regression and 6 for classification) and \emph{three types of Gaussian synthetic datasets} for regression, they are detailed in~\Cref{sec:dataset-description}.
 
\begin{itemize}
    \item \textbf{Convergence:} for \textbf{\textit{classification}}, we display in~\Cref{fig:adv-train-pgd}  suboptimality \emph{vs} iteration in the MNIST and in~\Cref{fig:convergence_classif} for Breast cancer and MAGIC-C datasets.  For \textbf{\textit{regression}}, in \Cref{fig:adv-train-icg} (\emph{bottom}) in the Abalone dataset and in~\Cref{fig:convergence_regression} for Wine and Diabetes.   In~\Cref{fig:baseline_gd}, we compare our algorithm with baseline implementation where we solve \Cref{eq:advtraining-closeform-linf} using standard gradient descent algorithm.
    \item \textbf{Execution time}: For \textbf{\textit{regression}}, we compare our execution time in the MAGIC dataset with CVXPY (\Cref{fig:adv-train-linear}(\emph{left})). In~\Cref{fig:adv-train-icg} (\emph{top}), we compare execution time for 100 iterations in synthetic Gaussian data for \textbf{\textit{regression}}. In~\Cref{fig:exectime_classification}, for our \textbf{\textit{classification}} methods.
    \item \textbf{Test performance:} In \Cref{fig:adv-train-linear}(\emph{left}), we compare the performance of $\ell_{\infty}$-adversarial training using default value against cross-validated lasso in different \textbf{\textit{regression}} \emph{real} datasets. An extended comparison in the same setup is provided in~\Cref{tab:performance_regression_all}. \Cref{fig:syntetic_data} illustrate for regression the performance in \emph{syntetic datasets} for varying density and of the number of relevant principal components. \Cref{fig:varying_n_vs_p} illustrate the performance over different fractions $n/p$. 
    We compare the performance of adversarial training for \textbf{\textit{classification}} in~\Cref{fig:performance_classification}.
    \item \textbf{Adversarial performance:}  we also compare (in the regression setting) adversarial training and Lasso when evaluated under adversarial attacks. For all scenarios of \Cref{fig:adv-train-linear}(a), we implemented and evaluated the same models under adversarial attacks. The results are displayed in the~\Cref{tab:adv_evaluation} for attacks with an adversarial radius of $\delta^{\rm eval} = 0.2$. \Cref{fig:syntetic_data_adv} illustrate for regression the performance in \emph{syntetic datasets} for varying adversarial radius of $\delta^{\rm eval}$. Overall, we observe that as expected adversarially trained models have a significant advantage over lasso in the presence of adversarial attacks.
\end{itemize}

\begin{table}[]
    \centering
    \caption{\emph{Adversarially evaluated performance in real datasets}. The setting is the same as~\Cref{fig:adv-train-linear}(a), but evaluated in an adversarially modified testset with an adversarial radius of $\delta^{\rm eval} = 0.2$.}
    \label{tab:adv_evaluation} 
    \begin{tabular}{lcc}
        \toprule
        Dataset & Adv Train  & Lasso CV  \\
        \midrule
        abalone     & 0.18  & -0.37 \\
        diabetes    & 0.20  & 0.04  \\
        diamonds    & 0.64  & 0.61  \\
        house sales & 0.43  & 0.39  \\
        pollution   & 0.64  & 0.51  \\
        us crime   & 0.47  & 0.40  \\
        wine        & 0.14  & 0.02  \\
        \bottomrule
    \end{tabular}\vspace{-10pt}
\end{table}

Our evaluation shows  reliable \emph{convergence} and efficient \emph{execution time} across various scenarios, including high-dimensional problems. Based on these experiments, we recommend the following general guidelines: for \textbf{\textit{regression}} problems, use the conjugate gradient method when we have \emph{both} a large number of parameters $p$ and training samples $n$, and the Cholesky method otherwise. For \textbf{\textit{classification}} problems, if the number of training samples $n$ is large, we suggest the use of SAGA, otherwise, we suggest using AGD with backpropagation.

The experiments also show good out-of-the-box \emph{test performance} in real datasets with default parameters.
Synthetic data experiments provide some intuition for the good practical performance of these methods. These experiments suggest adversarial training performs on par with a cross-validated lasso for data generated with \emph{sparse parameters} or \emph{spiked eigenvalue distributions} (see \Cref{fig:syntetic_data}), which is common in many real-world problems.

\begin{figure}\vspace{-10pt}
    \centering 
    \includegraphics[width=0.48\textwidth]{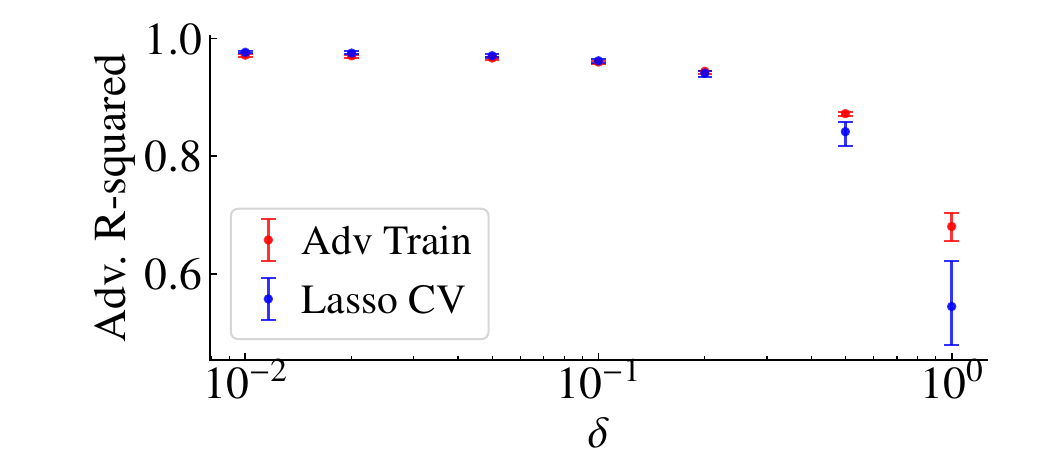}\vspace{-5pt}
    \caption{\emph{Adversarially evaluated performance in a syntetic dataset}. We consider spiked eigenvalue models, with varying evaluation adversarial radius $\delta^{\rm eval}$. }
    \label{fig:syntetic_data_adv}
    \vspace{-10pt}
\end{figure}
\vspace{-10pt}
\section{\uppercase{Conclusion}}
\label{conclusion}

We propose tailored solvers for linear  adversarial training. We hope this will enable new applications in large-scale problems like polygenic risk scores~\citep{torkamani_personal_2018}, proteomic analysis~\citep{safo_derivation_2023}, and Mendelian randomization~\citep{burgess_robust_2020}. 
Introducing a competing approach to Lasso that promotes sparsity in feature selection and delivers strong out-of-the-box performance with default values, even without using cross-validation.
% Making it a viable alternative to Lasso for feature selection without relying on cross-validation.
%Different reformulation strategies are needed for dealing with classification and regression. The $\eta$-trick does not apply immediately to classification problems, ant \Cref{thm:advtrain-classif-closeform} does to regression.
%Regression has stronger theoretically results that corroborate its good performance~\citep{ribeiro_regularization_2023}, still, we believe that our implementation for classification is also of interest.

It is a natural line to extend the method to different loss functions and other $\ell_p$-norms. Of special interest is the case of multiclass problems. Traditional strategies (i.e., one-vs-rest and one-vs-one) can be used here. However, providing a solver for an adversarial formulation of multinomial logistic regression requires extending \Cref{thm:rewriting-adv-error} and some of our other results to the multi-output case.

While we do not directly address neural network training, we hope our solutions can inspire work in this area. For example, given our promising results in the linear case, adapting the method proposed in \Cref{classification} for deep neural networks could be a valuable direction for future research.  Another promising approach to handling nonlinearity could involve studying adversarial kernel regression, where reformulating the problem as iterative ridge regression with the $\eta$-trick enables the application of the kernel trick to explore this line.

\acknowledgments{
TBS and DZ are financially supported by the Swedish Research Council, with the projects Deep probabilistic regression–new models and learning algorithms (project number 2021-04301) and Robust learning methods for out-of-distribution tasks (project number 2021-05022); and, by the Wallenberg AI, Autonomous Systems and Software Program (WASP) funded by Knut and Alice Wallenberg Foundation. TBS, AHR and DZ are partially funded by the Kjell och Märta Beijer Foundation. FB by the National Research Agency, under the France
2030 program with the reference ``PR[AI]RIE-PSAI” (ANR-23-IACL-0008). 
}

%\bibliography{refs}
%\include{checklist}

\newpage
\appendix

% Reset page

% Reset counters
\setcounter{equation}{0}
\renewcommand{\theequation}{S.\arabic{equation}}%

\setcounter{figure}{0}
\renewcommand{\thefigure}{S.\arabic{figure}}%

\setcounter{table}{0}
\renewcommand{\thetable}{S.\arabic{table}}%

\setcounter{section}{0}

\onecolumn
\pagenumbering{roman} 

\part{}\pagestyle{empty}

\aistatstitle{Supplementary Materials}\vspace{-40pt}
\parttoc % Insert the document TOC
\setcounter{page}{0 }
\newpage
\pagestyle{fancy}
\section{Adversarial training in linear models}
\label{background-apendix}
Here we present some basic results for adversarial training in linear models. 

\subsection{Reformulation for arbitrary loss functions}

\label{adv-train-reform}
The following theorem allows us to rewrite the inner maximization problem in~\Cref{eq:advtrain}. \Cref{thm:rewriting-adv-error} as we describe here is stated and proved in \citet{ribeiro_regularization_2023}. The same result but specifically for regression (and not in its general format) is presented and proved in~\citep{ribeiro_overparameterized_2023,xing_generalization_2021,javanmard_precise_2020}. And, for classification, in~\citep{goodfellow_explaining_2015,yin_rademacher_2019}.

\begin{theorem}[\citet{ribeiro_regularization_2023}, Section 8]
\label{thm:rewriting-adv-error}
Let $\loss: \R \times \R \rightarrow \R$, such that $\ell(y, \cdot)$ is convex and lower-semicontinuous (LSC) function for every $y$. Than for every $\delta\ge 0$,
\vspace{-1pt}
\begin{equation}
    \label{eq:adversarial-risk}
    \max_{\|\dx\| \le \delta} \loss\left(y, (\x + \dx)^\top \param\right) =  \max_{s \in \{-1, 1\}}  \loss\left(y, \x^\top \param  + \delta s  \|\param\|_* \right).
\end{equation}
\end{theorem}

We can find a closed-formula expression for
$s_* = \arg\max_{s \in \{-1, 1\}} \loss(\param^\top x  + \delta s  \|\param\|_* )$ and for  the righthand side of~\eqref{eq:adversarial-risk} both for \emph{regression} and \emph{classification} problems. 

\paragraph{Regression.} Let $\loss(y, \x^ \top \param) = f(|\x^\top \param - y|)$ where $f: \R^+ \rightarrow \R^+$ is a non-decreasing LSC convex function. If $f$ is LSC and convex so will be $\loss$ and the result of the theorem holds. Moreover, we have that $s_* = - \sign(y - \x^ \top \param) $  and 
\[\max_{\|\dx\| \le \delta}f(|(\x + \dx)^\top \param - y|) = f(|y - x^ \top \param| + \delta \|\param\|_*).\]
along the text we apply this result for $f(z) = z^2$, but other options are possible, i.e, $f(z) = |z| $

\paragraph{Classification.} Let $y \in \{-1, 1\}$ and $\loss(y, \x^ \top \param) = h(y(\x^\top \param))$ where $h$ is a non-increasing LSC function.  If $h$ is LSC and convex so will be $\loss$ and the result of the theorem holds.Moreover we have that $s_* = -y$ and we obtain:
\[\max_{\|\dx\| \le \delta}\ell(y((\x + \dx)^\top \param)) = \ell(y (\x^ \top \param) - \delta \|\param\|_*).\]
In the main text we focus on $h(z) = \log\left(1 + e^{-z}\right)$, but other alternatives are possible, i.e. $h(z) = (1 - y \hat{y})_+^2$.

\paragraph{Note.} In \Cref{thm:advtrain-classif-closeform}, we require $h$ to be smooth. Note that this is a sufficient condition for $h$ to be LSC but not necessary. Smoothness is required only for the second result of the proposition.

\subsection{Default choice of value}
\label{sec:default-value}
As we mentioned in the main text, adversarial training in linear regression can be formulated as the min-max problem: \[  \min_{\param}\frac1n \sum_{i=1}^\ntrain{\max_{\|\dx_i\| \le \delta} \big(y_i- (\x_i + \dx_i)^\top\param\big)}^2.\]

In this paper, we use the following rule-of-thumb for choosing the parameter for adversarial training
\[\delta^{\rm default} = P_{95}\big(\frac{\|X^{\top}\vv{\varepsilon}\|}{\|\vv{\varepsilon}\|}\big) \text{ for } \vv{\varepsilon}\sim \N(0, I), \] as the default option. Where we use $P_{95}$ to denote the 95\% percentile. Note that this result can be easily approximated using a Monte Carlo simulation. This rule of thumb is a consequence of two theoretical results that we summarize next and have been proved by \citet{ribeiro_regularization_2023}. 

\paragraph{Zero under sub-gaussian noise.} The first theoretical result is that this default value yields (with high-probability) zero parameters for the case the output is just sub-gaussian noise.
\begin{proposition}[Zero solution of adversarial training]
\label{thm:zero-solution}
The zero solution $\eparam = \vv{0}$ minimizes the adversarial training if and only if $\delta \ge  \frac{\|\X^\top \y\|}{\|\y\|_1}$.
\end{proposition}

The resul implies that if $\vv{y}=\vv{\varepsilon}\sim \N(0, I)$---i.e., the output is just noise and not connected with the input in any way---then with (high probability) we have $\delta > \frac{\|\X^\top \y\|}{\|\y\|_1}$ and the above proposition result implies a zero solution.

\textbf{Near-oracle performance.} The second result is that setting $\delta \propto \frac{\|\X^\top \ee \|_\infty}{\|\ee \|_1}$  yields near oracle performancefor $\ell_\infty$-adversarial training. Assume that the data was generated as: $y_i = \x_i^\top\param^* + \e_i$. Under these assumptions, we can derive an upper bound for the (in-sample) prediction error:
where $\param^*$ is the parameter vector used to generate the data.
\begin{theorem}
\label{thm:pred-error-slowrate}
Let $\delta > \delta^* = 3 \frac{\|\X^\top \ee \|_\infty}{\|\ee \|_1}$, the prediction error of $\ell_\infty$-adversarial training satisfies the bound:
\begin{equation}
\label{eq:pred-error-slowrate}
\tfrac{1}{n}\|\X (\eparam - \param^*)\|_2^2  \le  8 \delta\|\param^*\|_1  \left(\tfrac{1}{n}\|\ee\|_1  + 10 \delta \|\param^*\|_1\right).
\end{equation}
\end{theorem}
If we assume $\ee$ has i.i.d.~$\N(0, \sigma^2)$ entries and  that the matrix $\X$ is fixed with
$ \max_{j=1, \dots, m} \|\x_j\|_\infty \le M$ we obtain
${\frac{1}{n}\|\X (\eparam - \param^*)\|_2^2\lesssim  M\sigma  \sqrt{\nicefrac{(\log \nfeatures)}{\ntrain}}}$, which is a near-oracle performance for Lasso. See the details in \citep{ribeiro_regularization_2023}. Here in our default we do not include the factor 3 in our definition $\delta^{\rm default}$. The upper bound obtained in the Theorem is conservative approximations, so we found that the proposed $\delta^{\rm default}$ (without this factor of 3) works better in practice. 

\textbf{Improvements under restricted eigenvalue conditions:} \citet{xie_high-dimensional_2024} extend the results from \citet{ribeiro_regularization_2023} and show that if, additionally $X$ respect the restricted eigenvalue condition and the parameter is sparse with at most $s$ non-zero terms, for the similar choice of $\delta$ the following improved bound can be obtained (which is the same obtained for lasso with known variance under the same additional conditions):
\[{\frac{1}{n}\|\X (\eparam - \param^*)\|_2^2\lesssim  \sigma \frac{s(\log \nfeatures)}{\ntrain}}.\] 
We refer the reader to the paper for details.

\section{Classification}

\subsection{Projection step (Proof of \Cref{thm:proj})}
\label{proof-proj}
The projection of $(\widetilde{\param}, \tilde{t})$ into the convex set $C$, i.e., $\proj_C (\widetilde{\param}, \tilde{t})$, corresponds to solving the minimization problem
\[\min_{(\param, t)\in C}  \tfrac{1}{2}\|\param  - \widetilde{\param}\|^2_2 + \tfrac{1}{2}(t - \tilde{t})^2.\]
The solution of this constrained minimization problem is given in \Cref{thm:proj}) for $C = \{(\param, t): \rho t \ge \delta \|\param\|_2\}$ and for $C = \{(\param, t): \rho t \ge \delta \|\param\|_1 \}$ and is proved next.
\paragraph{$\ell_2$-adversarial attacks.}
For $C = \{(\param, t): \rho t\ge \delta \|\param\|_2\} = \{(\param, t): \rho^2 t^2\ge  \delta^2 \|\param\|_2^2\}$, the  Lagrangian of this minimization problem is:
\[\mathcal{L}(\param, t, \lambda)  = \tfrac{1}{2}\|\param  - \widetilde{\param}\|^2_2 + \tfrac{1}{2}(t - \tilde{t})^2 + \lambda (\delta^2\|\param\|_2^2 - \rho^2 t^2).\]
This Lagrangian is differentiable and its partial derivatives are
\begin{subequations}
    \label{eq:subderivative_l2}
\begin{align}
\nabla_{\param} \mathcal{L}(\param, t, \lambda) &= (\param  - \widetilde{\param}_i) + \lambda \delta^2 \param \\
\frac{\partial\mathcal{L}}{\partial t}(\param, t, \lambda) &= (t - \tilde t) - \lambda \rho^2 t
\end{align}
\end{subequations}
From KKT conditions, $(\param, t)$ is a solution iff:
\begin{enumerate}
    \item Stationarity:
    \begin{subequations}
    \label{eq:stationarity_l2}
        \begin{align}
            0 &= \nabla_{\param}  \mathcal{L}(\param, t, \lambda)\\
            0 &= \frac{\partial\mathcal{L}}{\partial t}(\param, t, \lambda).
        \end{align}
    \end{subequations}
    \item Slackness:
    $\lambda (\rho t - \delta\|\param\|_2) = 0.$
    \item Primal feasibility:
    $\rho t \ge \delta \|\param\|_2$
    \item Dual feasibility: $\lambda \ge 0$
\end{enumerate}
Combining \eqref{eq:subderivative_l2} and \eqref{eq:stationarity_l2}:
\begin{subequations}
    \label{l2proj}
    \begin{align}
        \param  &=\frac{1}{1 + \delta^2 \lambda} \widetilde{\param} \\
        t &= \frac{1}{1 - \rho^2 \lambda} \tilde{t}
    \end{align}
\end{subequations}
If $\lambda = 0$, than $(\widetilde{\param}, \tilde{t})$ is a solution. Now, if $\lambda > 0$, we have $\rho t = \delta \|\param\|_2$ and therefore:
\[\frac{\rho}{1 - \rho^2 \lambda} \tilde{t} = \frac{\delta}{1 + \delta^2 \lambda} \|\widetilde{\param}\|_2 \]
Solving the above equation for $\lambda$ we obtain that \[\lambda = \frac{1}{\rho \delta} \frac{\delta\|\widetilde{\param}\|_2 - \rho \tilde{t}}{\rho\|\widetilde{\param}\|_2 + \delta\tilde{t}}\]
plugging $\lambda$ back into \eqref{l2proj} obtain the desired projection.

\paragraph{$\ell_{\infty}$-adversarial attacks.} Now, for $C = \{(\param, t): \rho t\ge \delta \|\param\|_1\}$, the Lagrangian of the minimization problem is :
\[\mathcal{L}(\param, t, \lambda)  = \tfrac{1}{2}\|\param  - \widetilde{\param}\|^2_2 + \tfrac{1}{2}(t - \tilde{t})^2 + \lambda (\delta \|\param\|_1 - \rho t),\]
and its subderivative is
\begin{subequations}
    \label{eq:subderivative}
\begin{align}
\partial_{\beta_i} \mathcal{L}(\param, t, \lambda) &= (\beta_i  - \widetilde{\beta}_i) + \delta \lambda \partial|\beta_i| \\
\partial_{t}\mathcal{L}(\param, t, \lambda) &= (t - \tilde t) - \rho \lambda
\end{align}
\end{subequations}
where:
$$\partial |\beta| = \begin{cases}
    {1} & \beta > 0\\
    {-1} &\beta < 0\\
    [-1, 1] &\beta = 0
\end{cases}$$
From KKT conditions, $(\param, t)$ is a solution iff:
\begin{enumerate}
    \item Stationarity:
    \begin{subequations}
    \label{eq:stationarity}
        \begin{align}
            0 &\in \partial_{\beta_i} \mathcal{L}(\param, t, \lambda), \forall i\\
            0 &\in \partial_{t} \mathcal{L}(\param, t, \lambda).
        \end{align}
    \end{subequations}
    
    \item Slackness:
    $\lambda (\rho t - \delta \|\param\|_1) = 0.$
    \item Primal feasibility:
    $\rho t \ge \delta \|\param\|_1$
    \item Dual feasibility: $\lambda \ge 0$
\end{enumerate}

\noindent
Combining \eqref{eq:subderivative} and \eqref{eq:stationarity} we obtain that $(\beta_i, t)$ is a solution iff:
\begin{subequations}
    \begin{align}
        \widetilde{\beta}_i &= \beta_i + \delta \lambda \partial|\beta_i|\\
       \tilde t &=  t - \rho\lambda
    \end{align}
\end{subequations}
Or, inverting the above equations:
\begin{subequations}
    \begin{align}
        \beta_i &= \text{sign}(\widetilde{\beta}_i)\left(|\widetilde{\beta}_i|  - \delta \lambda \right)_+\\
       t &=   \tilde t + \rho \lambda.
    \end{align}
\end{subequations}
If $\lambda = 0$, than $(\widetilde{\param}, \tilde{t})$ is the solution. Now, for the case $\lambda > 0$, we have $\|\param\|_1 = \frac{\rho}{\delta} t$, hence:
$$\sum_{i = 1}^\nfeatures \left(|\widetilde{\beta}_i|  - \lambda\delta  \right)_+  = \|\param\|_1 =  \frac{\rho}{\delta} t =  \frac{\rho}{\delta}  \tilde t + \frac{\rho^ 2}{\delta}  \lambda.$$
The above equations define $\lambda$ uniquely. On the one hand, the lefthand size is decreasing and assume the value $\|\widetilde{\param}\|_1$ for $\lambda = 0$. The righthand side is strictly increasing and assumes the value $\tilde{t}$ for $\lambda=0$ and grows indefinitely as $\lambda \rightarrow \infty$. Since $(\widetilde{\param}, \tilde{t}) \not \in C$ we have that $\delta \|\widetilde{\param}\|_1 > \rho \tilde{t}$, and we can always find a unique solution to the equation.  We note that our derivation is similar to what is used to project a point into the $\ell_1$-ball (See for instance: \url{https://angms.science/doc/CVX/Proj_l1.pdf}).

\subsection{Closed formula solution for FGSM }
\label{closed-form-fgsm}
The next proposition establishes a closed formula solution that is obtained when we use FGSM to choose the adversarial attack.

\begin{proposition}
\label{thm:fgsm}
Let $h$ be non-increasing, convex. 
$$\max_{\|\dx_i\|\le \delta}h\big(y_i(\x_i + \dx_i)^\top\param\big) = h\big(y_i(\x_i + \dx_i)^\top\param\big),$$
and the maximum of  is achieve  for ${\dx_i = \delta \sign \nabla_x J(\x_i)}$. 
\end{proposition}

\begin{proof}
    See~\Cref{adv-train-reform} for a proof of that $\max_{\|\dx_i\|\le \delta}h\big(y_i(\x_i + \dx_i)^\top\param\big) = h\big(y_i(\x_i + \dx_i)^\top\param\big)$. Now we prove the second statement. Let $J_i(x_i) =  h\big(y_i\x_i^\top\param\big)$, we have that $\nabla_x J(\x_i) = h'(y_i\x_i^\top\param\big) y_i \param$. Now, since, $h$ is non increasing we have that $h'\le 0$ for any argument. Hence $\dx_i = \delta \sign \nabla_x J(\x_i) = - \delta  y_i \sign{\param}$ and the result follows.
\end{proof}

\subsection{Improvements on the basic algorithm}
\label{adversarial-train-logist-implementations}
Here we provide pseudo-algorithms for the methods described in~\Cref{sec:improvements}.
\begin{algorithm}[H]
\caption{Augmented formulation with projected   GD\\
\phantom{Algorithm X} and backtracking LS}\label{alg:pgd-with-ls}
\textit{Choose} initial  step size $\gamma^{(0)}$ \\
$\vv{w}^{0} \leftarrow (\param^{(0)},t^{(0)})$\\
\textbf{for $k = 1, 2 \dots$:}
\begin{enumerate}
    \item[]$\vv{w}^{+} \leftarrow  \proj_C\left(  \vv{w}^{(k-1)} - \gamma^{(k-1)} \nabla_{\param} \mathcal{R}(\vv{w}^{(k-1)})\right)$
    \item[] \emph{\# Backtraking line search:}
    \item[] \textbf{while }$\mathcal{R}(\vv{w}^+) > \widetilde{\mathcal{R}}_{\vv{w}}(\vv{w}^+; \gamma)$:
    \begin{enumerate}
    \item[] $\gamma^{(k-1)}\leftarrow  \gamma^{(k-1)} / 2$
    \item[] $\vv{w}^{+} \leftarrow  \proj_C\left(  \vv{w}^{(k-1)} - \gamma \nabla_{\param} \mathcal{R}(\vv{w}^{(k-1)})\right)$
    \end{enumerate}
    \item[] $ \vv{w}^{(k)}  = \vv{w}^{+}$, $\gamma^{(k)} \leftarrow  \gamma^{(k-1)}$
\end{enumerate}
\end{algorithm}

\begin{algorithm}[H]
\caption{Augmented formulation with projected SGD}\label{alg:psgd}
\textit{Choose}  step size sequence $\gamma^{(k)}, k = 1, 2, \dots$ \\
$\vv{w}^{0} \leftarrow  (\param^{(0)},t^{(0)})$ \\
\textbf{for $k = 1, 2 \dots$:}
\begin{enumerate}
\item[] \emph{Sample} $i^k\in\{1,\dots, \ntrain\}$.
    \item[]$\vv{w}^{(k)} \leftarrow  \proj_C\left(  \vv{w}^{(k-1)} - \gamma^{(k)} \nabla f_{i^k}(\vv{w}^{(k-1)})\right)$
\end{enumerate}
\end{algorithm}

\begin{algorithm}[H]
\caption{Augmented formulation with accelerated \\
\phantom{Algorithm X} projected  GD}\label{alg:apgd}
\textit{Choose} initial  step size $\gamma^{(0)}$ \\
$\vv{w}^{0} \leftarrow (\param^{(0)},t^{(0)})$\\
\textbf{for $k = 1, 2 \dots$:}
\begin{enumerate}
    \item[] $\vv{z}^{(k)} \leftarrow \vv{w}^{(k)} + \mu^{(k)}  \left(\vv{w}^{(k)} - \vv{w}^{(k-1)}\right),$
    \item[] $\vv{w}^{(k)} \leftarrow \proj_C\left(\vv{z}^{(k-1)} - \gamma \nabla_{w} \mathcal{R}(\vv{z}^{(k-1)})\right)$
    \item[ ] $\mu^{(k)} \leftarrow \frac{\alpha_{k}-1}{\alpha_{k+1}}$ for $\alpha_{k} = \frac{1 + \sqrt{1 + 4 {\alpha_{k}^2}}}{2}$.
\end{enumerate}
\end{algorithm}

\begin{algorithm}[H]
\caption{Augmented formulation with projected SAGA}\label{alg:psaga}
\textit{Choose}  step size $\gamma$ \\
$\vv{w}^{0} \leftarrow  (\param^{(0)},t^{(0)})$, $v_{i} \leftarrow 0$ for $i = 1, \dots n$\\
\textbf{for $k = 1, 2 \dots$:}
\begin{enumerate}
\item[] \emph{Sample} $i^k\in\{0, 1,\dots, \ntrain\}$.
\item[] $v_{\text{old}} \leftarrow v_{i^k}$
\item[] $v_{i^k}\leftarrow \nabla f_{i^k}(\vv{w}^{(k-1)})$

\item[] $g_k\leftarrow v_{i^k} - v_{\text{old}} + \bar{g}_k$
    \item[]$\vv{w}^{(k)} \leftarrow  \proj_C\left(  \vv{w}^{(k-1)} - \gamma g_k\right)$
    \item[] $\bar{g}_{k+1} \leftarrow \bar{g}_{k} + \frac{1}{n}(v_{i^k} - v_{\text{old}})$ 
\end{enumerate}
\end{algorithm}

\newpage

\section{Regression}

\subsection{Proof of joint convexity (in the proof of \Cref{thm:advtraining-closeform})}
\label{joint-convexity}
% https://math.stackexchange.com/questions/1403116/joint-convexity-proof
The proof of joint convexity follows from the technical lemma:

\begin{proposition}
    \label{thm:joint-convexity}
    The following functions:
    \begin{enumerate}
        \item $f_1(x, Y) = x^T A^{\top} Y^{-1} A x$;
        \item $f_2(x, Y) = b^\top Y^{-1} A x$;
        \item $f_3(Y) = b^\top Y^{-1}b$;
    \end{enumerate}
    are convex in the domain $\mathcal{D}_1 = \{(x, Y): x\in \R^p, Y\in \R^{p\times p}| Y \succ 0\}\R^p$; and also if we restrict the domain to $\mathcal{D}_2= \{(x, Y): x\in \R^p, Y\in \R^n |  Y = \text{diag}(y), y > 0\}$.
\end{proposition}

\begin{proof}
    Let us start by proving that $f_1$ is convex in $\mathcal{D}_1$. By definition, $f_1$ is convex iff its epigraph:
    \[\small \mathcal{E}= \{(x, Y, z): x\in \R^n\text{ and }Y \succ 0 \text{ and } x^T A^{\top} Y^{-1} A x \le z\}\]
    is convex. Now, using the Schur complement rule for semidefinite matrices, given that $Y \succ 0$:
    \[z - x^T A^{\top} Y^{-1} A x \ge 0  \iff 
   \begin{bmatrix}
        Y & \vv{0} \\
        \vv{0} & z - x^T A^{\top} Y^{-1} A x
    \end{bmatrix} \succcurlyeq 0\iff 
    \begin{bmatrix}
        Y & A x \\
        x^T A^{\top} & z 
    \end{bmatrix} \succcurlyeq  0 
    \]
    
    Therefore, the epigraph is the intersection of one strict LMI and one non-strict LMI, both describing convex sets. As the intersection of two convex sets, the epigraph is convex, and so the function is convex. The same argument applies for the domain $\mathcal{D}_2$. The proof for $f_2$ and $f_3$ follow the same steps.
\end{proof}

Let us call $\mm{H}^{(i)} = \text{diag}(\vv{\eta}^{(i)})$ we have:
\[G_\infty(\param, \mm{H}^{(0)}, \cdots, \mm{H}^{(n)}) = \left(\y_i - \X \param\right)^{\top}(\mm{H}^{(0)})^{-1}\left(\y_i - \X \param\right) + \delta \sum_{i=1}^n \param^\top(\mm{H}^{(i)})^{-1}\param.\]
If we expand and rearrange we get functions that can be written in the format of function $f_1$, $f_2$ and $f_3$ from \Cref{thm:joint-convexity} and hence their sum is convex.

\subsection{Proof of \Cref{thm:eta-trick} ($\eta$-trick)}
\label{proof-eta-trick}

Let start proving the case, where $a_t>0$ for all $t$ and $\epsilon = 0$.
We can rewrite the original optimization problem from \Cref{thm:eta-trick} as
\begin{equation*}
\min_{\boldsymbol{\eta}>0} f(\boldsymbol{\eta})  \myeq \sum_{t=0}^{T-2} \frac{a_t^2}{\eta_t} + \frac{a_{T-1}^2}{ 1- \sum_{t=1}^{T-2} \eta_t}.
\end{equation*}
Here $f$ is convex by \Cref{thm:joint-convexity}, hence $\widehat{\vv{\eta}}$ is a minimum iff  $\nabla f(
\vv{\eta}) = \textbf{0}$ . Now, 
\begin{equation*}
\frac{\partial f}{\partial \eta_t} (\boldsymbol{\eta})  = -\frac{a_t^2}{\eta_t^2} + \frac{a_{T-1}^2}{(1- \sum_{t=1}^{T-2} \eta_t)^2},
\end{equation*}
which is equal to zero for every $t$ iff:
\begin{equation*}
\eta_t = (1- \sum_{t=1}^{T-2} \eta_t) \frac{a_t}{a_{T-1}}.
\end{equation*}
We can sum the above equation for $t = 1$ to $T-1$ and rearranging we obtain:
\[ (1- \sum_{t=1}^{T-2} \eta_t) \frac{1}{a_{T-1}} = \frac{1}{\sum_{t=1}^{T-1} a_t},\]
therefore:
\begin{equation*}
\eta_t = \frac{a_t}{\sum_{t=1}^{T-1} a_t}.
\end{equation*}
Now, we focus on proving that the solution is unique if $a_t> 0$ for all $t$. We have that
\begin{equation*}
\frac{\partial^2 f}{\partial \eta_t \partial \eta_{\ell}} (\boldsymbol{\eta})  =
\begin{cases}
    2 \frac{a_t^2}{\eta_t^3} + 2 \frac{a_{T-1}^2}{(1- \sum_{t=1}^{T-2} \eta_t)^3}, & \text{if } \ell = t\\
    2 \frac{a_{T-1}^2}{(1- \sum_{t=1}^{T-2} \eta_t)^3}& \text{if }  \ell \not= t
\end{cases}
\end{equation*}
And,
\begin{equation*}
\frac{1}{2}\nabla^2 f (\boldsymbol{\eta}) = \text{diag}(\frac{a_1^2}{\eta_1^3}, \dots, \frac{a_1^2}{\eta_1^3}) + \frac{a_{T-1}^2}{(1- \sum_{t=1}^{T-2} \eta_t)^3} \textbf{1} \textbf{1}^{\top} 
\end{equation*}
Hence: $\frac{1}{2}\nabla^2 f (\boldsymbol{\eta}) \succcurlyeq \min_t \frac{a_t}{\eta_t^3}\mm{I}$ and the problem is strongly convex as long as $a_t> 0$ for all $t$, and therefore it has a unique solution.

Now to prove the case where $\epsilon> 0$, we can just apply the  result we just proved with $\tilde a_t  = \sqrt{a_t^2 + \epsilon}$ instead of $a_t$ to obtain:
\begin{equation*}
\left(\sum_{t=1}^T \tilde a_t\right)^2 = \min_{\boldsymbol{\eta} \in \Delta_T} \left(\sum_{t=0}^{T-1} \frac{a_t^2 + \epsilon}{\eta_t} \right).
\end{equation*}
Expanding and manipulating we obtain the desired result.

\subsection{Results for $\ell_2$-adversarial attacks}
\label{ell2-attacks}

A similar algorithm can be derived for $\ell_2$-adversarial training. We can rewrite the cost function as:
\begin{equation}
  \label{eq:advtraining-closeform-l2}
  \min_{\param}\sum_{i=1}^n\left(|y_i - \x_i^\top\param| + \delta\|\param\|_2\right)^2.
\end{equation}
\noindent
Using~\Cref{thm:eta-trick} (with $T=2$ and $\epsilon =0$) for the $i^{\mathrm{th}}$ term in the sum:
\[\left(|y_i - \x_i^\top\param| + \delta\|\param\|_2\right)^2 = \min_{\boldsymbol{\eta}^{(i)} \in \Delta_2}\frac{\left(y_i - \x_i^\top\param \right)^2}{\eta^{(i)}_0}  + \frac{\|\param\|_2^2}{\eta^{(i)}_1} \]
hence, we can rewrite the problem as
\begin{equation}
\label{eq:advtraining-closeform-l2-eta-trick}
\min_{\param} \min_{\substack{\boldsymbol{\eta}^{(i)}\in \Delta_2 \\i=1, \dots, n}}\mathcal{G}_2 (\param, \boldsymbol{\eta}^{(1)}, \dots\boldsymbol{\eta}^{(n)}),
\end{equation}
where:
$$\mathcal{G}_2  = \sum_{i=1}^n\frac{\left(y_i - \x_i^\top\param\right)^2}{\eta^{(i)}_0} + \delta\left(\sum_{i=1}^n\frac{1}{\eta^{(i)}_1}\right)\|\param\|_2^2.$$
Again, blockwise coordinate descent in $\mathcal{G}_2$ yields~\Cref{alg:irrr}. Now, the function $\mathtt{UpdateWeights}$ compute
\begin{align*}
    w_i = \frac{1}{\eta^{(i)}_0},~~~
    \gamma_j =  \sum_{i=1}^n\frac{1}{\eta^{(i)}_1}, j= 1, \dots, \nfeatures.
\end{align*}
Again, in practice, we might want to apply \Cref{thm:eta-trick} with a $\epsilon>0$ but small. The derivation would be the same as above, but we would end up with an additional penalty term $\epsilon \sum_i \Big(\frac{1}{\eta^{(0)}} + \frac{1}{\eta^{(1)}}\Big)$  added to the definition of $\mathcal{G}_2$.

\subsection{Conjugate gradient implementation}
Here we provide pseudo-algorithms for the conjugate gradient implementation described in~\Cref{sec:improvements-irr}.
\begin{algorithm}
    \caption{Iterative Conjugate Gradient}\label{alg:icg}
    \textbf{Initialize:}
    \begin{itemize}
        \item[] sample weights $w_i\leftarrow 1, i = 1, \dots, \ntrain$.
        \item[] parameter weights $\gamma_i\leftarrow 1, i = 1, \dots, \nfeatures$.
    \end{itemize}
    \textbf{Repeat:}
    \begin{enumerate}
    \item Define:$$\mm{A} = \mm{X}^\top \mm{W} \mm{X} + \delta \mm{\Gamma},~~ \vv{b} = \mm{X}^\top \mm{W}  \vv{y}$$
    \item \emph{Approximately solve} reweighted ridge regression using
    \begin{equation*}
        \widehat{\param}^{(\ell+1)} \leftarrow \mathtt{ConjugateGradient}\big(\mm{A}, \vv{b}, \widehat{\param}^{(\ell)}\big)
    \end{equation*}
    \item \emph{Update} weights
    $$\boldsymbol{w}^{(\ell+1)}, \boldsymbol{\gamma}^{(\ell+1)}\leftarrow \mathtt{UpdateWeights}(\widehat{\param}^{(\ell+1)})$$
    \item \emph{Quit} if \texttt{StopCriteria}.
    \end{enumerate}
\end{algorithm}

\begin{algorithm}[H]
    \caption{Preconditioned Conjugate Gradient Descent}\label{alg:cg}
    \textbf{Initialize:}
    \begin{itemize}
        \item[] Given $\param^{(0)}$ and preconditioner $\mm{M}$
        \item[] Compute $\mathbf{r} \leftarrow \mm{A}\param^{(0)} - \vv{b}$.
        \item[] Solve $\mm{M} \vv{z} = \vv{r}$ for $\vv{z}$
        \item[] Set $\vv{p} \leftarrow -  \vv{z}$, $k \leftarrow 0$
    \end{itemize}
    \textbf{Repeat for $k = 1, \dots$:}
    \begin{align}
     \alpha &\leftarrow (\vv{r}^\top \vv{z})/ (\vv{p}^\top \mm{A}\vv{p}) \\
     \param^{(k)}  &\leftarrow \param^{(k-1)} + \alpha \vv{p}\\
     \vv{r}_+ &\leftarrow \vv{r}  + \alpha \mm{A}\vv{p}\\
     \text{Solve }&\mm{M}\vv{z}_+ = \vv{r}_+\text{for }\vv{z}_+\\
     \xi &\leftarrow (\vv{r}_+^\top \vv{z}_+)/ (\vv{r}^\top \vv{z})\\
     \vv{p} &\leftarrow - \vv{z}_+ \xi \vv{p}\\
     \vv{r}, \vv{z} &\leftarrow \vv{r}_+, \vv{z}_+
    \end{align}
\end{algorithm}

\subsection{Numerical complexity: our implementation and CVXPY internals}
\label{CVXPY}

Bellow we highlight some key differences between CVXPy and our implementation:

\begin{itemize}
    \item \textbf{CVXPY} internally converts linear adversarial regression to a quadratic programming problem of dimension $2(m+n)$ with $n$ constraint. It uses OSQP solver (https://osqp.org/docs/index.html). Which is powered by an ADMM algorithm~\citep{boyd_distributed_2011} and at each interaction a quadratic problem of dimension $2(m+n) n$, having thus a cost per interaction $4n (m+n)^2 $.
    \item \textbf{Our method} for regression uses an iterative ridge regression method. Each interaction here has a cost $\min(m^2n, mn^2)$. This type of algorithm also usually converges in a few iterations\cite{bach_eta-trick_2019}.
\end{itemize}
CVXPY has other solvers and for some of them, it might result in a conic program instead (for instance, when we used MOSEK solver in Table~\ref{tab:mosek}). Adversarially training in classification results in similar problems for CVXPY. Overall, the main challenge of using disciplined convex programming here is that it fails to use the structure of the problem. And, it ends up with a problem that scales with the largest dimensions $(m, n)$ instead of the smallest one. It also fails to use cheap approximations for the subproblem as we did in~\Cref{sec:improvements-irr}.

\begin{table}[t]
    \centering
\begin{tabular}{c|m{1.4cm}m{1.4cm}m{1.4cm}}
\toprule
\# params & CVXPY \hbox{(default)}  & CVXPY (MOSEK) & Ours \\
\midrule
30 & 0.07 & 0.09 & 0.07 \\
100 & 0.2 & 0.3 & 0.1 \\
300 & 0.9 & 1.0 & 0.4 \\
1000 & 8.1 & 5.8 & 1.1 \\
3000 & 42.4 & 15.2 & 3.4 \\
10000 & 141.0 & 31.6 & 12.9 \\
30000 & 164.6 & 100.3 & 41.9 \\
\bottomrule
\end{tabular}
\caption{\textbf{CVXPY using MOSEK}. Comparison of runtimes (in seconds) for the different methods. MOSEK might be a bit faster in some situations, still, we don't use it in other experiments, because \emph{it requires a license and could harm the easy reproducibility of our code}.}
\label{tab:mosek}
\end{table}

\section{Numerical Experiments}
\label{sec:extra-numerical-experiments}
 We describe the datasets used in \Cref{sec:dataset-description}. We describe additional numerical experiments aim to evaluate the \textbf{convergence}; \textbf{execution time}; and, \textbf{test performance}  in \Cref{convergence}, \Cref{execution_time}, \Cref{test_performance}.

\subsection{Datasets}
\label{sec:dataset-description}
We consider the following datasets in the analysis: {MAGIC $(\ntrain\mathord{=}504, \nfeatures\mathord{=}55067)$;} {Diabetes $(\ntrain\mathord{=}442, \nfeatures\mathord{=}10)$;}  {Abalone $(\ntrain\mathord{=}4177, \nfeatures\mathord{=}8)$;} {Wine  $(\ntrain\mathord{=}4898, \nfeatures\mathord{=}11)$;}  {Polution $(\ntrain\mathord{=}60, \nfeatures\mathord{=}15)$}; {US Crime $(\ntrain\mathord{=}1994, \nfeatures\mathord{=}127)$}; {House Sales $(\ntrain\mathord{=}21613, \nfeatures\mathord{=}21)$}; {and, Diamonds $(\ntrain\mathord{=}54000, \nfeatures\mathord{=}9)$}. And, for \textbf{\textit{classification}}:  Breast cancer $(\ntrain\mathord{=}442, \nfeatures\mathord{=}10)$; MNIST $(\ntrain\mathord{=}60000, \nfeatures\mathord{=}784)$; MAGIC-C $(\ntrain\mathord{=}504, \nfeatures\mathord{=}55067)$;  {Iris $(\ntrain\mathord{=}4, \nfeatures\mathord{=}150)$}; {Blood  transfusion $(\ntrain\mathord{=}748, \nfeatures\mathord{=}5)$}; and, {Heart failure $(\ntrain\mathord{=}299, \nfeatures\mathord{=}13)$}.  Many of them are available (and we provide the IDs when applicable) in the OpenML \citep[]{vanschoren_openml_2014} or in the UCI Machine learning repository databases. The links for these two data repositories are: \url{www.openml.org} and \url{archive.ics.uci.edu}. 

\paragraph{Classification.}

\begin{itemize}
    \item \textbf{MNIST:} \citep[OpenML ID=554]{lecun_gradientbased_1998} The MNIST database of handwritten digits has a training set of $n\mathord{=}$60000 examples, and a test set of 10,000 examples. It is a subset of a larger set available from NIST. The digits have been size-normalized and centered in a fixed-size image, the image has $\nfeatures\mathord{=}784$ features.

    \item \textbf{Breast cancer:} \citep{street_nuclear_1993} Features ($p\mathord{=}30$) are computed from a digitized image of a fine needle aspirate (FNA) of a breast mass.  They describe the characteristics of the cell nuclei present in the image. This is a classification task that try to predict breast cancer diagnosis for $n\mathord{=}569$ patients.

    \item \textbf{MAGIC C:} \citep{scott_limited_2021} We consider exactly the same dataset as the MAGIC dataset ($p\mathord{=}55067$, $n\mathord{=}504$) described for regression but here we try to predict a binary phenotype instead of a continuous one.

    \item \textbf{Heart failure:} \citep[UCI ID=519]{chicco_machine_2020} This dataset contains the medical records of $n\mathord{=}$299 patients who had heart failure, collected during their follow-up period, where each patient profile has $p\mathord{=}$13 clinical features. 
     
    \item \textbf{Iris} \citep[OpenML ID=41078]{fisher_use_1936} It comprises $p\mathord{=}$4 measurements on each of $n\mathord{=}$150 plants of three different species of iris. It was first used as an example by R. A. Fisher in 1936, and can now be found in multiple online archives and repositories

    \item \textbf{Blood trasfusion} \citep[OpenML ID=1464]{yeh_knowledge_2009}  Data taken from the Blood Transfusion Service Center in Hsin-Chu City in Taiwan.  It considers $p\mathord{=}5$ features: months since the last donation,  total number of donations, total blood donated in c.c., months since the first donation, and blood donation in March. For $n\mathord{=}748$ donors in the database.
\end{itemize}

\paragraph{Regression.}
\begin{itemize}
    \item \textbf{Diabetes:} \citep{efron_least_2004} The dataset has $\nfeatures\mathord{=}10$ baseline variables (age, sex, body mass index, average blood pressure, and six blood serum measurements), which were obtained for $n\mathord{=}442$ diabetes patients. The model output is a quantitative measure of the disease progression.

    \item \textbf{Abalone:} (OpenML ID=30, UCI ID=1) Predicting the age of abalone from $p\mathord{=}8$ physical measurements.  The age of abalone is determined by cutting the shell through the cone, staining it, and counting the number of rings through a microscope -- a boring and time-consuming task.  Other measurements, which are easier to obtain, are used to predict the age of the abalone. It considers $n\mathord{=}4417$ examples. 
    % UCI 1
    % OpenML 30
    
    \item \textbf{Wine quality:} \citep[UCI ID=186]{cortez_modeling_2009} A large dataset ($n\mathord{=}4898$) with white and red "vinho verde" samples (from Portugal) used to predict human wine taste preferences. It considers $p\mathord{=}11$ features that describe physicochemical properties of the wine.
    % UCI 186
    
    \item \textbf{Polution:} \citep[OpenML ID=542]{mcdonald_instabilities_1973} Estimates relating air pollution to mortality. It considers $p\mathord{=}15$ features including precipitation, temperature over the year, percentage of the population over 65, besides socio-economic variables, and concentrations of different compounds in the air. In total it considers It tries to predict age-adjusted mortality rate in $n\mathord{=}60$ different locations.
    % OpenML: 

    \item \textbf{Diamonds:} (OpenML ID=42225) The dataset contains the prices and other attributes of almost $n\mathord{=}54000$ diamonds. It tries to predict the price from $p\mathord{9}$ features that include: carat weight, cut quality, color, clarity, xyz lengths, depth and width ratio.

    \item \textbf{MAGIC:} \citep{scott_limited_2021} We consider Diverse MAGIC wheat dataset from the National Institute for Applied Botany. 
The dataset contains the whole genome sequence data and multiple phenotypes for a population of $n\mathord{=}504$ wheat lines.  We use a subset of the genotype to predict one of the continuous phenotypes.
We have integer input with values indicating whether each one of the 1.1 million nucleotides differs or not from the reference value. Closely located nucleotides tend to be correlated and we consider a pruned version provided by~\citep{scott_limited_2021} with $p=55067$ genes. 

    \item \textbf{US crime:} \citep[OpenML ID=42730,  UCI ID=182]{redmond_data-driven_2002} This dataset combines $p\mathord{=}127$ features that come from socio-economic data from the US Census, law enforcement data from the LEMAS survey, and crime data from the FBI for $n\mathord{=}1994$ comunities. The task is to predict violent crimes per capita in the US as target. 

    \item \textbf{House sales:} (OpenML ID=42731) This dataset contains house sale prices for King County, which includes Seattle. It includes homes sold between May 2014 and May 2015. It contains $p\mathord{=}19$ house features, along with $n\mathord{=}21613$ observations. 
\end{itemize}

\paragraph{Syntetic datasets.}

\begin{itemize}
\item \textbf{Isotropic.}  We consider Gaussian noise and covariates: $\epsilon_i \sim \N(0, \sigma^2)$ and $\x_i \sim \N(0, r^2 \mm{I}_\nfeatures)$  and the output is computed as a linear combination of the features contaminated with  additive noise: 
\[y_i = \x_i^\trnsp \param+ \epsilon_i.\]
In the experiments, unless stated otherwise, we use  the parameters $\sigma=1$ and $r=1$. We also consider the features are sampled from a gaussian distribution $\param \sim \N(0, 1/\sqrt{p}\mm{I}_\nfeatures)$

\item \textbf{Spiked covariance.} We use the model described in~\citet[Section 5.4]{hastie_surprises_2022}, also referred as "latent feature models". The features~$\x$ are noisy observations of a lower-dimensional subspace of dimension~$d$. A vector in this \textit{latent space} is represented by $\vv{z} \in \R^d$. This vector is indirectly observed via the features $\x \in \R^p$ according to
\[\x = \mm{W}\vv{z} + \vv{u},\]
where $\mm{W}$ is an $\nfeatures \times \inpdim$ matrix, for $\nfeatures \ge \inpdim$. We assume that the responses are described by a linear model in this latent space
\[y = \vv{\theta}^\top \vv{z} + \xi,\]
where $\xi \in \R$ and $\vv{u}\in \R^\nfeatures$ are mutually independent noise variables. Moreover, $\xi \sim \N(0, \sigma_{\xi}^2)$ and $\vv{u} \sim \N\left(0, \mm{I}_{\nfeatures}\right)$. We consider the features in the latent space to be isotropic and normal $\vv{z} \sim  \N\left(0, \mm{I}_{\inpdim}\right)$ and choose $\mm{W}$ such that its columns are orthogonal, $\mm{W}^\top \mm{W} = \frac{\nfeatures}{\inpdim} \mm{I}_{\inpdim}$, where the factor $\frac{\nfeatures}{\inpdim}$ is introduced to guarantee that the signal-to-noise ratio of the feature vector $\x$ (i.e. $\frac{\|\mm{W} \vv{z}|_2^2}{\|\vv{u}\|_2^2}$) is kept constant. In the experiments, unless stated otherwise, we use  the parameters $\sigma_\xi=1$ and the latent dimension fixed $d=1$.
\item \textbf{Sparse vector.}  Here we consider $\param^*$ is sparse vector with $s$ non-zero features. We call the support $S(\param) = \{j\in \{1, \cdots, d\}: \beta_j \not=0\}^*$, the support vector is sampled uniformly at random. The value of the non-zero parameters are sampled from a normal distribution $\N(0, 1/\sqrt{s}\mm{I}_s)$. The output is computed as a linear combination of the features contaminated with  additive noise: 
$y_i = \x_i^\trnsp \param ^*+ \epsilon_i.$
and consider Gaussian noise and covariates: $\epsilon_i \sim \N(0, \sigma^2)$ and $\x_i \sim \N(0, r^2 \mm{I}_\nfeatures)$.
\end{itemize}

\subsection{Comments about the implementation}
\label{implementation}

The current implementation has some \emph{limitations}: currently, it works only with datasets that fit into memory, though this constraint could be addressed in future updates. Additionally, the implementation does not utilize GPU acceleration, which could potentially enhance performance. While our current implementation is based on NumPy, exploring implementations using deep learning libraries could offer further advantages. Moreover, our implementation implements the gradient computation efficiently (using Cython). But the projection in a less efficient way (using pure Python). Here GD and AGD compute only one projection per iteration, while SGD and SAGA compute one projection per sample. We leave implementation improvements for future work. Partially due to this stochastic algorithms (SGD and SAGA) have a time per iteration that increases quite fast (see \Cref{fig:exectime_classification})

\subsection{Convergence}
\label{convergence}

\begin{figure}[H]
    \vspace{-10pt}
    \centering
    \subfloat[step size]{\includegraphics[width=0.5\textwidth]{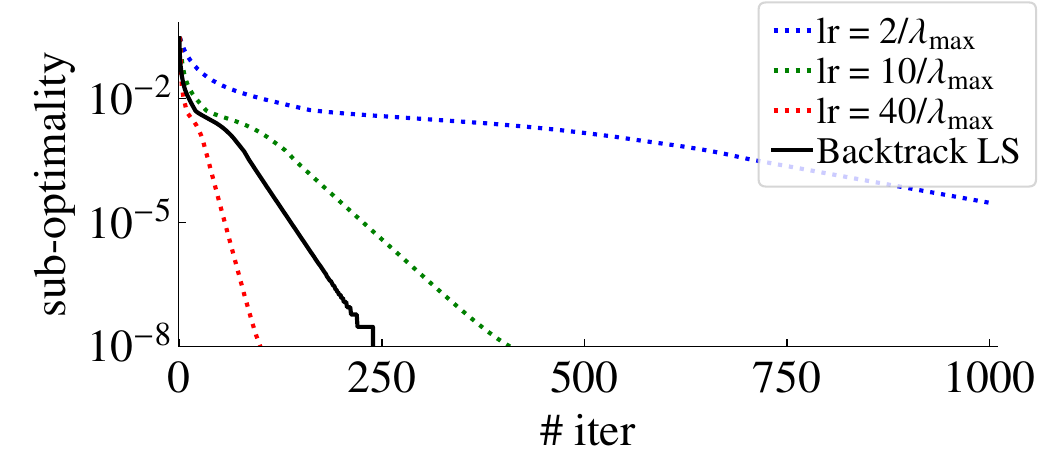}}
    \subfloat[acceleration]{\includegraphics[width=0.5\textwidth]{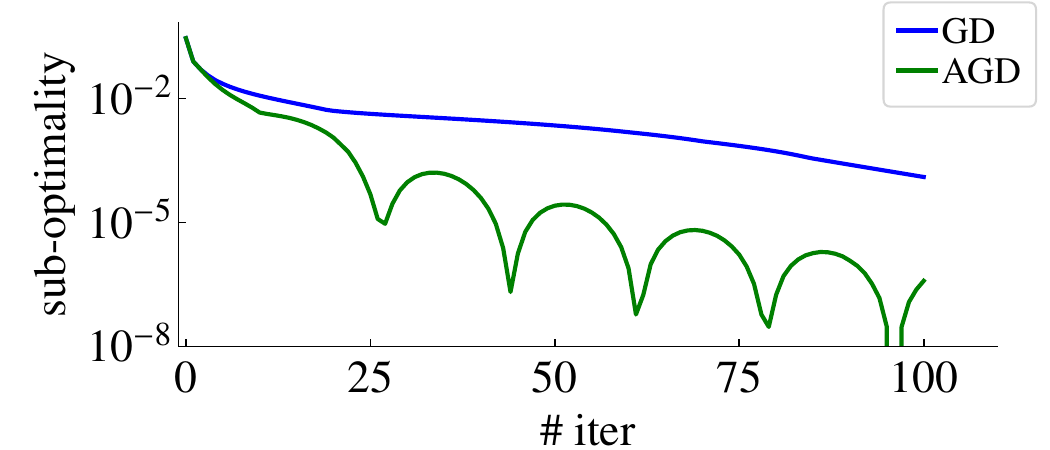}}\\
    \subfloat[stochastic]{\includegraphics[width=0.5\textwidth]{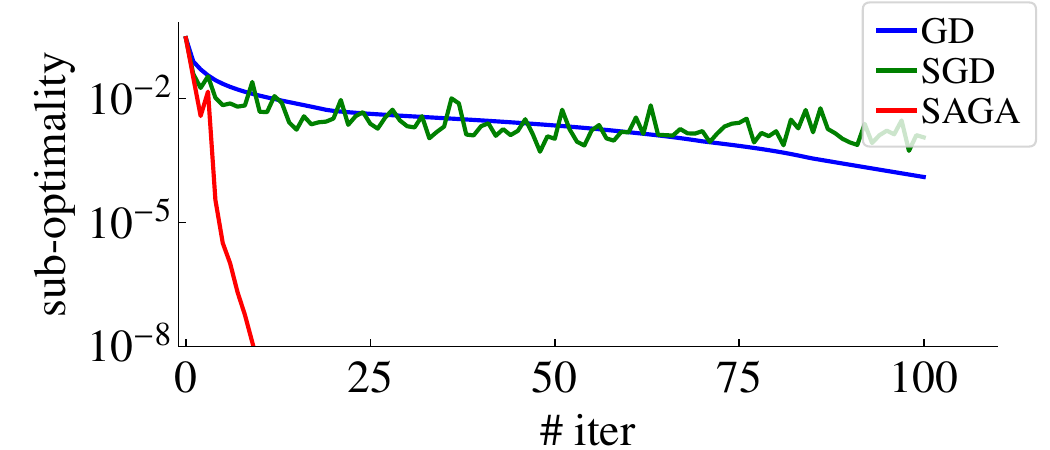}}

    \caption{\emph{Convergence classification.} Sub-optimality \emph{vs} the number of iterations in the breast cancer dataset. In (a), we show the results for different step sizes. In (b) we show the effect of momentum (with line search implemented for both methods). In (c) we show the effect of using a stochastic algorithm. Results for $\ell_\infty$-adv. training with $\delta=0.01$.}
    \label{fig:convergence_classif}\vspace{-10pt}
\end{figure}

\begin{figure}[H]
    \vspace{-10pt}
    \centering
    \subfloat[diabetes]{\includegraphics[width=0.5\textwidth]{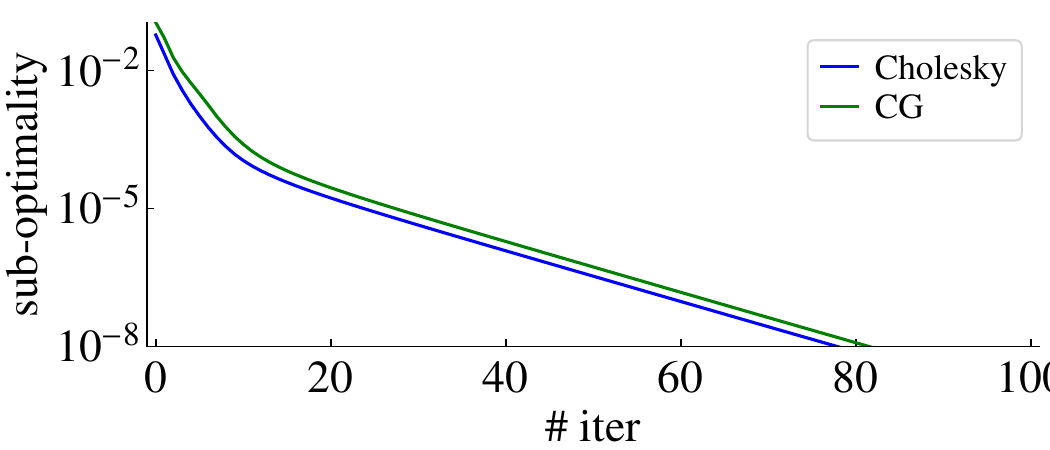}}
    \subfloat[wine]{\includegraphics[width=0.5\textwidth]{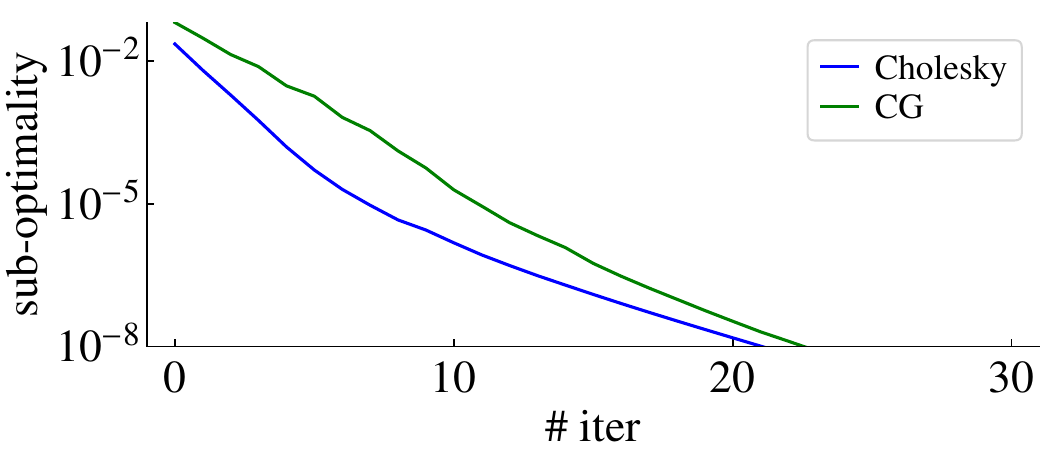}}
    \caption{\emph{Convergence regression.} Sub-optimality \emph{vs} the number of iterations in two datasets: (a) Diabetes; and (b) Wine. We use default $\delta$ described in \Cref{sec:default-value}}
    \label{fig:convergence_regression}\vspace{-10pt}
\end{figure}

\begin{figure}[H]
    \centering
    \includegraphics[width=0.47\linewidth]{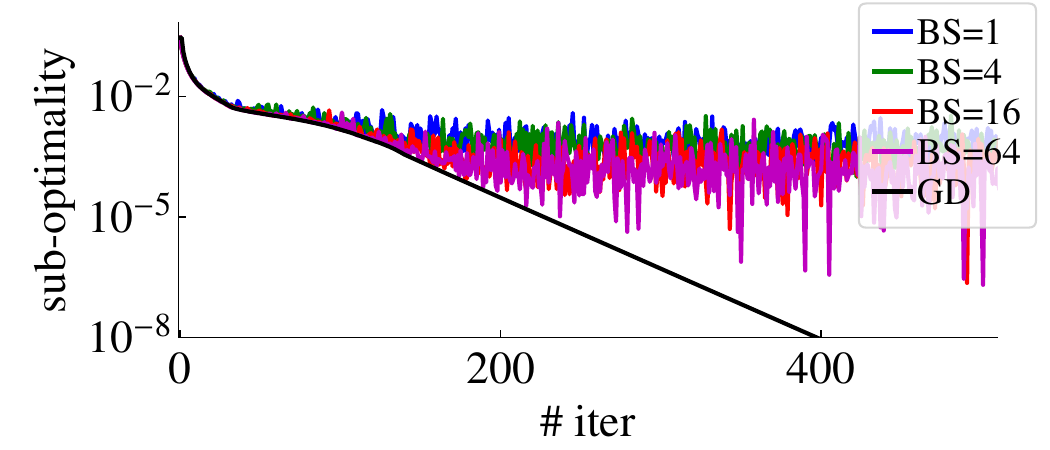}
    \caption{\emph{Varying batch size.} Sub-optimality vs number of iterations for SGD with increasing batch size. The learning rate is constant for all examples. Results for adversarial training with $\delta = 0.01$. Here one iteration is a full pass through the dataset.}
    \label{fig:varying-batch-size}
\end{figure}

\begin{figure}[H]
    \centering
    \includegraphics[width=0.6\linewidth]{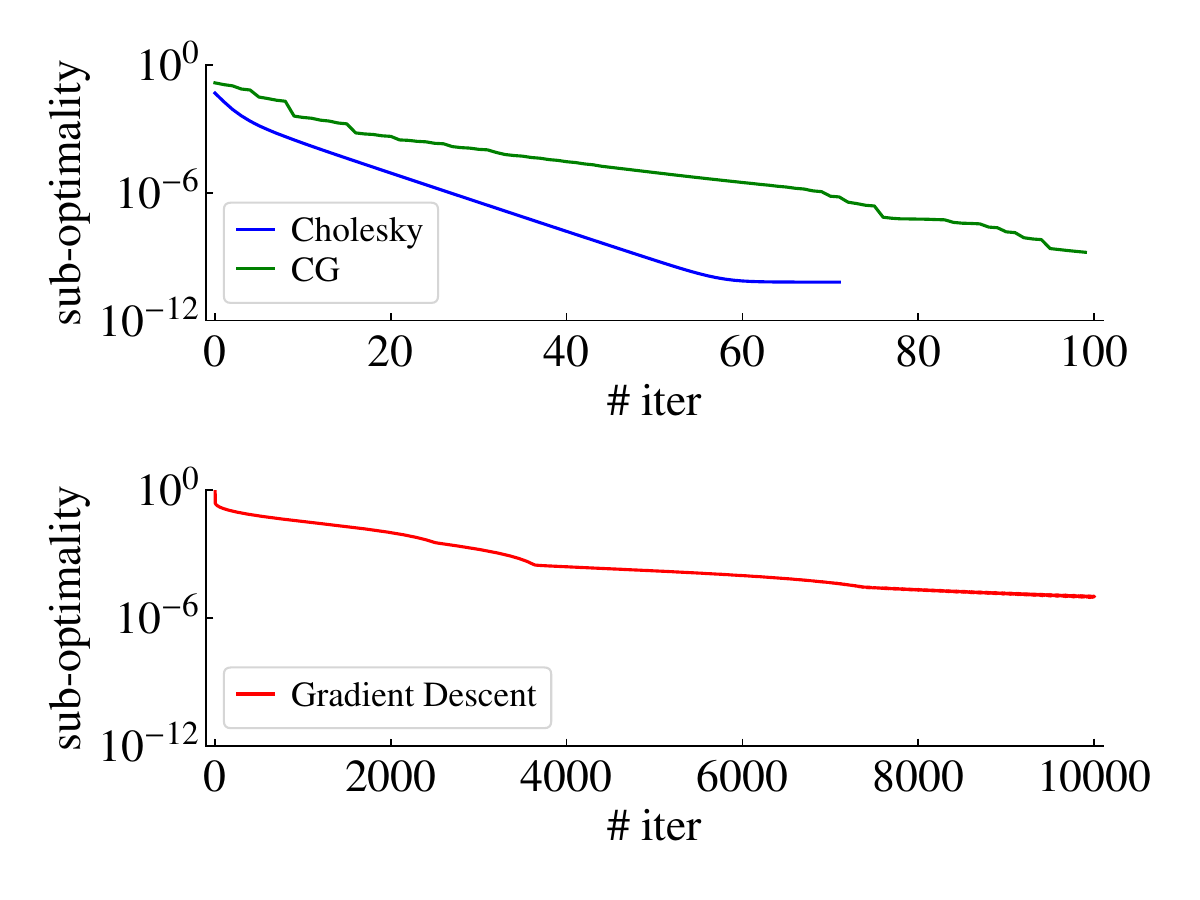}
    \caption{\emph{Gradient descent baseline for regression.} Comparison of our regression algorithm with the gradient descent applied to Equation~\eqref{eq:advtraining-closeform-linf}. Overall, gradient descent in the closed formula solution is not very efficient. In the example above, in 10000 iterations it achieves a suboptimality of 1e-6, while, the proposed conjugate gradient implementation (which has a similar number of gradient evaluations per iteration) achieves a suboptimality of 1e-10 in ~100 iterations. We highlight that here the function is not smooth and the transformation that we used for classification to make the cost function smooth does not directly translate to regression. Hence we cannot obtain the same improved rates that were obtained there.}
    \label{fig:baseline_gd}
\end{figure}

\subsection{Execution time}
\label{execution_time}

\begin{figure}[H]
    \vspace{-10pt}
    \centering
    \includegraphics[width=0.5\textwidth]{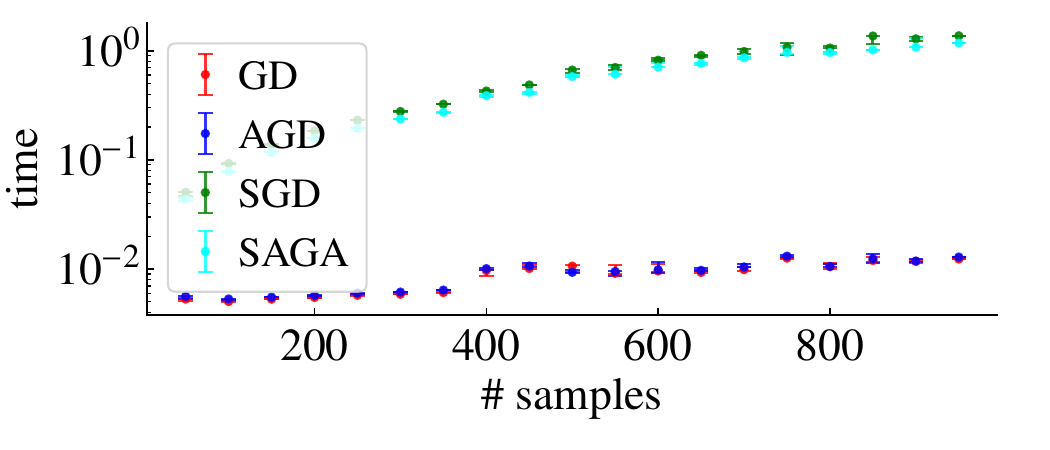}
    \caption{\emph{Execution time classification.} Execution time of 100 epochs using the different proposed algorithms. Here one iteration of the SGD or SAGA is equivalent to a full pass through the data. We use an isotropic synthetic dataset $y_i = \text{sign}(\x_i^{\top} \param + \epsilon_i).$ with a fixed ratio between the number of training samples and the number of features $\frac{n}{p} = 10$.}
    \label{fig:exectime_classification}\vspace{-10pt}
\end{figure}

\subsection{Test performance}
\label{test_performance}

\begin{table}[H]
 \footnotesize
 \caption{\emph{Performance of in \textbf{regression} tasks.}  We compare $\ell_\infty$ and $\ell_2$-adversarial training with default lasso and ridge regression using scikit-learn values and with the regularization parameter set with cross-validations. We also consider, BIC and AIC criteria for lasso. As well as the nonlinear methods: Multilayer perception with one layer and Gradient Boost.}
 \subfloat[Coefficient of determination $R^2$ (higher is better).]{
 \resizebox{\columnwidth}{!}{
% Generated from Pandas 
\begin{tabular}{lrrrrrrrr}
\toprule
method &  advtrain\_linf &  lasso\_cv &  lasso &  advtrain\_l2 &  ridge &  ridgecv &  gboost &  mlp \\
\midrule
Abalone     &           0.49 &      0.52 &  -0.00 &         0.50 &   0.50 &     0.52 &    0.53 & 0.58 \\
Diabetes    &           0.34 &      0.38 &  -0.04 &         0.37 &   0.37 &     0.38 &    0.30 & 0.20 \\
Diamonds    &           0.89 &      0.89 &  -0.00 &         0.89 &   0.89 &     0.89 &    0.97 & 0.97 \\
House sales &           0.68 &      0.68 &  -0.00 &         0.68 &   0.68 &     0.68 &    0.87 & 0.87 \\
Polution    &           0.73 &      0.78 &  -0.05 &         0.71 &   0.61 &     0.73 &    0.63 & 0.55 \\
US crime    &           0.64 &      0.64 &  -0.00 &         0.64 &   0.64 &     0.64 &    0.66 & 0.44 \\
Wine        &           0.28 &      0.29 &  -0.00 &         0.28 &   0.29 &     0.29 &    0.38 & 0.35 \\
\bottomrule
\end{tabular}
%%%%%%%%%%%%%%%%%%%%%%%%%
}
}
 \label{tab:r2_all}
 \footnotesize
  \subfloat[Root mean squared error (lower is better).]{
\footnotesize
\resizebox{\columnwidth}{!}{
% Generated from Pandas 
\begin{tabular}{lrrrrrrrr}
\toprule
method &  advtrain\_linf &  lasso\_cv &  lasso &  advtrain\_l2 &  ridge &  ridgecv &  gboost &  mlp \\
\midrule
Abalone     &           0.72 &      0.70 &   1.01 &         0.71 &   0.71 &     0.70 &    0.69 & 0.65 \\
Diabetes    &           0.73 &      0.71 &   0.92 &         0.71 &   0.71 &     0.71 &    0.75 & 0.80 \\
Diamonds    &           0.34 &      0.34 &   1.00 &         0.34 &   0.34 &     0.34 &    0.17 & 0.16 \\
House Sales &           0.57 &      0.57 &   1.01 &         0.57 &   0.57 &     0.57 &    0.36 & 0.36 \\
Polution    &           0.49 &      0.44 &   0.97 &         0.51 &   0.59 &     0.49 &    0.57 & 0.63 \\
US crime    &           0.60 &      0.59 &   0.99 &         0.60 &   0.59 &     0.59 &    0.58 & 0.74 \\
Wine        &           0.86 &      0.85 &   1.01 &         0.86 &   0.85 &     0.85 &    0.79 & 0.82 \\
\bottomrule
\end{tabular}
%%%%%%%%%%%%%%%%%%%%%%%%%
}}
 \label{tab:performance_regression_all}
\end{table}

\begin{figure}[H]
    \vspace{-10pt}
    \centering
    \subfloat[Isotropic]{\includegraphics[width=0.5\textwidth]{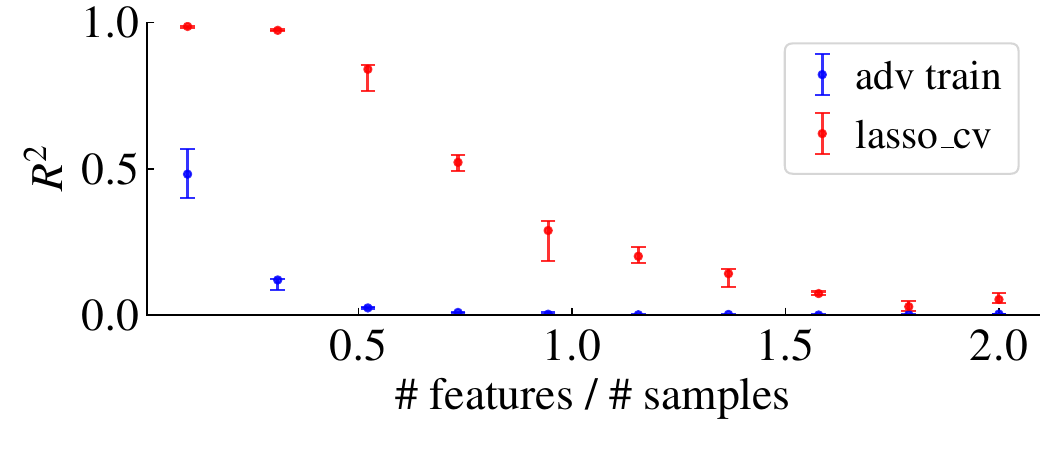}}
    \subfloat[Spiked eigenvalues]{\includegraphics[width=0.5\textwidth]{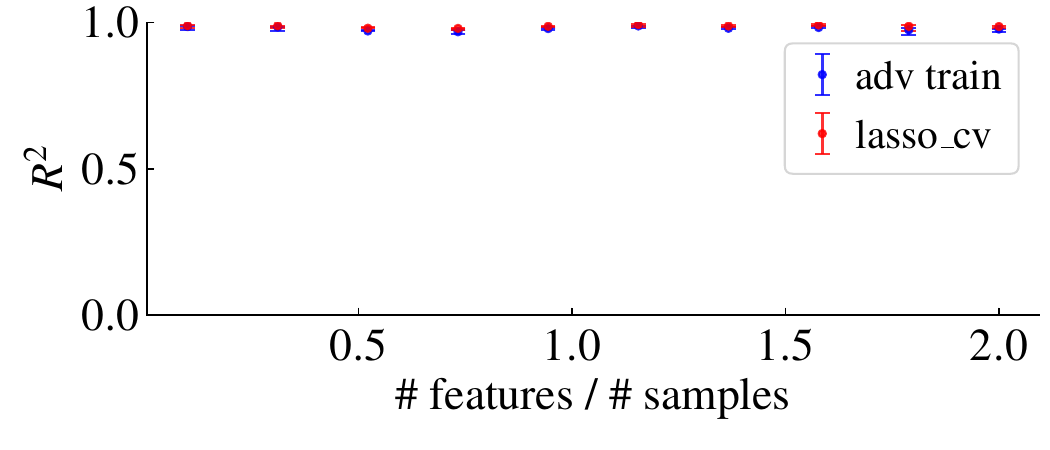}}\\
    \subfloat[Sparse parameter vector]{\includegraphics[width=0.5\textwidth]{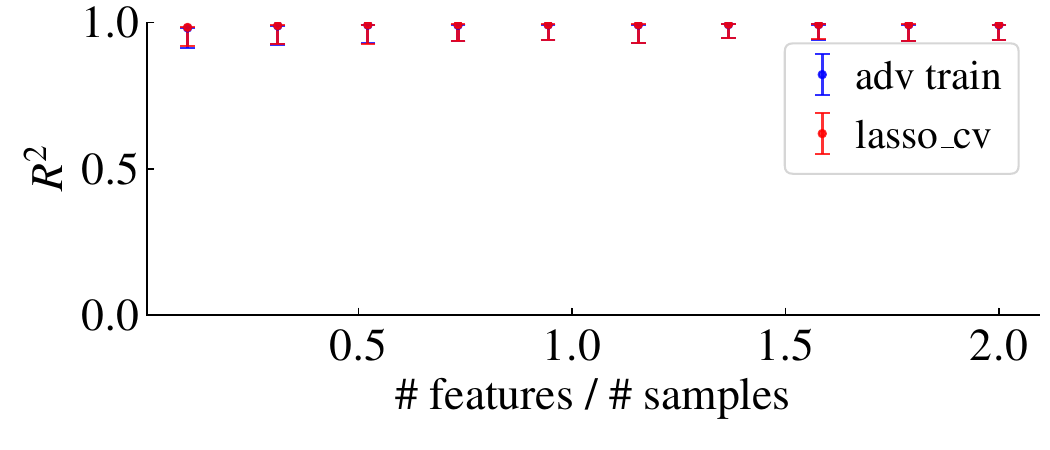}}
    \caption{\emph{Adversarial \textbf{regression} in linear regression.} Coefficient of determination $R^2$ vs the ratio between number of features $p$ and number of training samples $n$.}
    \label{fig:varying_n_vs_p}\vspace{-10pt}
\end{figure}

\begin{figure}[H]
    \vspace{-10pt}
    \centering
\includegraphics[width=0.5\textwidth]{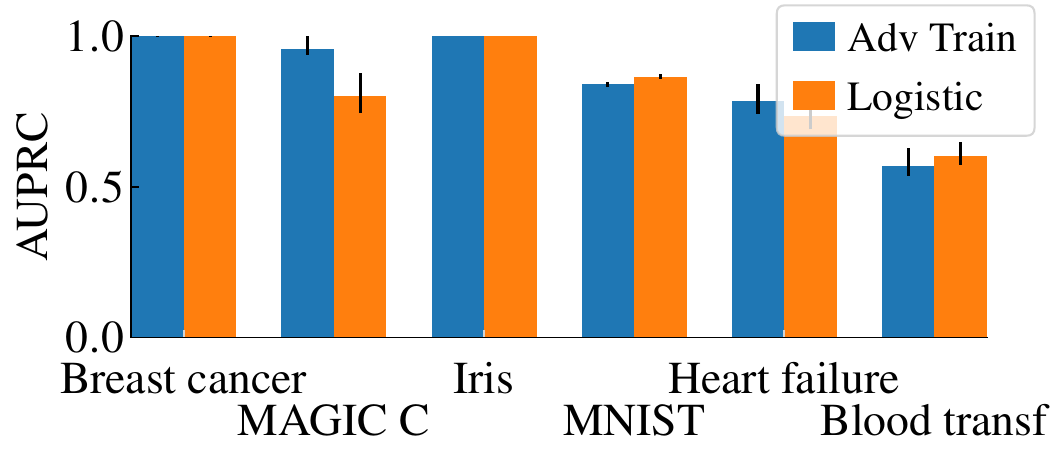}
    \caption{\emph{Adversarial \textbf{classification} in linear regression.} We give the Area Under the Precision-Recall Curve (AUPRC). The default value $\delta$ we described in~\Cref{sec:default-value} here is use for classification datasets, most of the reasoning does not necessarily apply here, but we do it for comparison. We compare with Logistic regression with default $\ell_2$ regularization also with default value (and not with Cross-Validation). The performance is mostly comparable. Being significantly better only for the MAGIC-C dataset. }
    \label{fig:performance_classification}\vspace{-10pt}
\end{figure}

\end{document}